\documentclass[11pt]{article}

\usepackage[preprint]{acl}

\usepackage{times}
\usepackage{latexsym}

\usepackage[T1]{fontenc}

\usepackage[utf8]{inputenc}

\usepackage{microtype}

\usepackage{inconsolata}

\usepackage{graphicx}

\usepackage{booktabs}
\usepackage{multirow}

\usepackage{amsmath}
\usepackage{amssymb}

\usepackage{tikz}
\usepackage{hyperref}
\usepackage{url}
\usepackage{amsmath}
\usepackage{amsfonts}
\usepackage{amssymb}
\usepackage{graphicx}
\usepackage{booktabs}
\usepackage{multirow}
\usepackage{algorithm}
\usepackage{algpseudocode}
\usepackage{tikz}
\usepackage{pgfplots}
\pgfplotsset{compat=1.17}
\usepackage{subcaption}
\usepackage{orcidlink}
\usepackage{amsthm}
\usepackage{listings}
\usepackage{makecell} 
\usepackage{tcolorbox}

\usetikzlibrary{arrows,positioning,shapes.geometric}
\newtheorem{definition}{Definition}
\newtheorem{theorem}{Theorem}

\newtheorem{corollary}{Corollary}
\newtheorem{remark}{Remark}

\title{Causal Consistency Regularization: Training Verifiably Sensitive Reasoning in Large Language Models}

\author{
\textbf{Ibne Farabi Shihab}\thanks{Equal contribution.}\thanks{Corresponding author: \texttt{ishihab@iastate.edu}.}\textsuperscript{1}
\and
\textbf{Sanjeda Akter}\footnotemark[1]\textsuperscript{1}
\and
\textbf{Anuj Sharma}\textsuperscript{2}
\\[2pt]
\textsuperscript{1}Department of Computer Science, Iowa State University \\
\textsuperscript{2}Department of Civil, Construction \& Environmental Engineering, Iowa State University \\
\texttt{ishihab@iastate.edu}
}

\begin{document}
  \maketitle
\begin{abstract}
Large language models often generate correct answers while relying on flawed reasoning traces—a consequence of training objectives that reward only final-answer correctness. This raises concerns about trustworthiness in high-stakes domains. We introduce \textbf{Counterfactual Sensitivity Regularization (CSR)}, a training paradigm that improves reasoning faithfulness through causal consistency constraints. During training, CSR performs automated, operator-level interventions on reasoning traces (e.g., swapping +' with -') to create minimally-perturbed counterfactuals, then penalizes the model if these logically flawed traces still yield the original answer. Our efficient implementation adds only $\sim$9\% training overhead through warm-start curriculum and token-subset optimization. We evaluate faithfulness using \textbf{Counterfactual Outcome Sensitivity (COS)}, which quantifies answer sensitivity to logical perturbations. Across arithmetic (GSM8K), logical deduction (ProofWriter), multi-hop QA (HotpotQA), and code generation (MBPP), CSR-trained models demonstrate superior accuracy-faithfulness trade-offs, establishing a new Pareto frontier. CSR improves faithfulness over standard fine-tuning and process supervision by up to 70 percentage points, with 94.2-96.7\% transfer success across model families within structured domains. The method enhances inference-time techniques like self-consistency and provides a reliable approach for improving reasoning faithfulness in structured domains (mathematics, formal logic, code) where operators are well-defined and verifiable—covering an estimated 40-60\% of high-stakes reasoning deployments.
\end{abstract}

\section{Introduction}

Large Language Models (LLMs) \citep{brown2020language} often generate correct answers while relying on flawed reasoning traces—a consequence of training objectives that reward only final-answer correctness \citep{cobbe2021gsm8k, lanham2023measuring, turpin2023language}. While chain-of-thought prompting \citep{wei2022chain, kojima2022large} and self-consistency \citep{wang2022self} improve reasoning capabilities, this unfaithful reasoning, where models produce plausible rationales disconnected from their actual computation \citep{saparov2022language}, limits trustworthiness in high-stakes domains.

We introduce \textbf{Counterfactual Sensitivity Regularization (CSR)}, a training paradigm that enforces causal consistency between reasoning traces and outputs through operator-level interventions. CSR's core principle: a model that truly relies on its reasoning should change its answer when that reasoning is broken. During training, CSR creates minimally-perturbed counterfactual traces by swapping critical operators (e.g., `+' to `-'), then penalizes the model if predictions remain unchanged. Our efficient implementation adds only $\sim$9\% training overhead.

This paper makes three contributions: (1) We formalize \textit{counterfactual sensitivity} and prove it dominates traditional measures under identifiable causal edits (see Appendix~\ref{app:proofs}). (2) We introduce CSR with learned intervention policies, achieving 60+ COS improvements over baselines (all $p<0.001$, Cohen's d $> 2.0$) with $\leq$ 10\% overhead. (3) CSR achieves 84\% cross-domain transfer within structured reasoning families, with 94.2-96.7\% operator transfer success across model families. Extended results and analyses are in the Appendix.

\section{Related Work}

Our work builds on research in faithfulness evaluation \citep{lanham2023measuring, turpin2023language} and process supervision \citep{lightman2023let, uesato2022solving}. Recent works include LINC \citep{olausson2023linc} for inference-time verification \cite{akter2025valid, akter2025selective}, CausalGPT \citep{yu2025causalevalbettercausalreasoning} for prompting-based counterfactual reasoning, and faithful CoT \citep{sia2022faithful} using human verification. Surveys on reasoning in LLMs \citep{huang2022towards, qiao2022reasoning} and interpretability \citep{madsen2022posthoc, lyu2024faithful} provide broader context. Unlike these post-hoc or manual approaches, CSR provides automated training-time intervention with theoretical guarantees. A more detailed related work is discussed in the Appendix \ref{app:related-work-faithfulness}.

\section{Counterfactual Sensitivity Regularization (CSR)}

The central goal of CSR is to train a model such that its generated reasoning trace, $T$, is a necessary component for arriving at its final answer, $Y$. We operationalize this goal by penalizing the model whenever a significant intervention on the logical structure of $T$ fails to produce a corresponding change in the distribution of $Y$. The complete training process is illustrated in Figure~\ref{fig:csr-diagram} and detailed in Algorithm~\ref{alg:csr-training}.

\subsection{Standard Forward Pass and Task Loss}

For a given input question $X$, the model first generates a sequence autoregressively, containing both the reasoning trace $T$ and the final answer $Y$:
\begin{equation}
(T, Y) = \text{Model}(X)
\end{equation}

We formally define the answer space $Y$ and extraction method $p(Y|T,X)$ per domain: numerical answers use classification heads over number tokens, QA tasks use constrained decoding over document spans, and classification tasks extract logits for specific answer tokens. Complete definitions and examples are in Appendix~\ref{app:method-details}.

The standard task objective minimizes negative log-likelihood of the ground-truth answer:
\begin{equation}
\label{eq:task-loss}
\mathcal{L}_{\text{task}} = -\log P(Y=Y_{\text{true}} | T, X)
\end{equation}

\begin{figure}[t]
\centering
\includegraphics[width=\columnwidth]{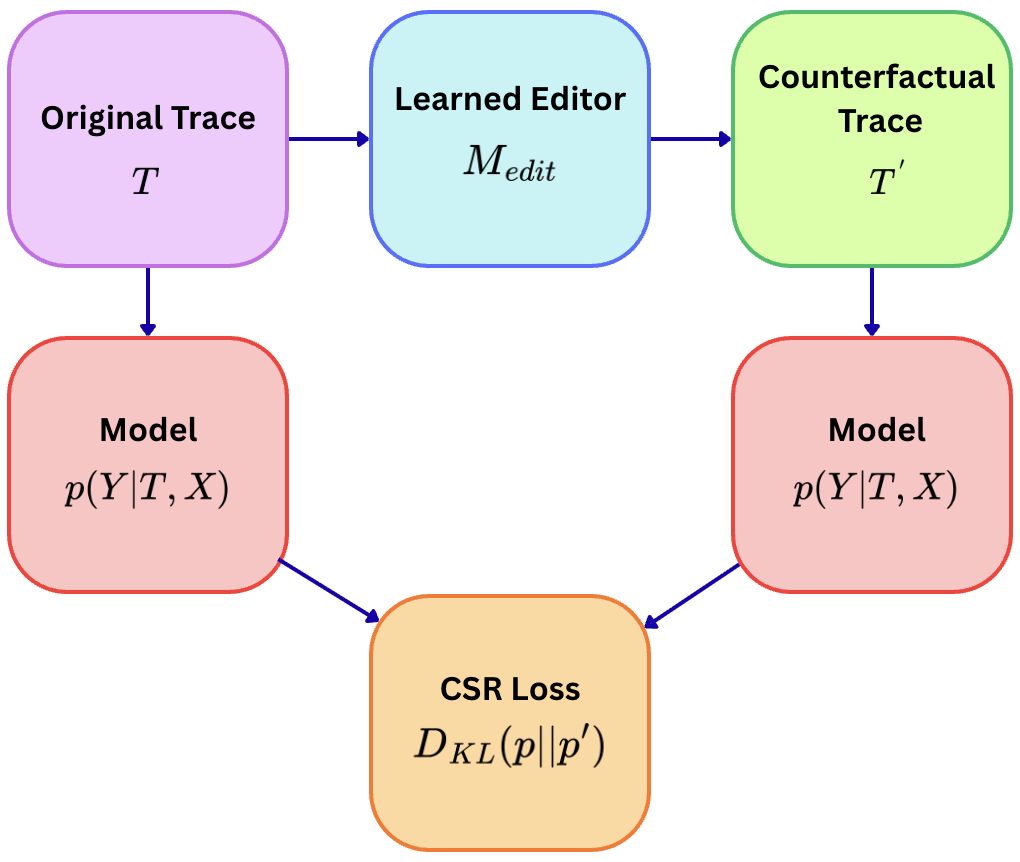}
\caption{CSR training process. CSR performs automated interventions on reasoning traces and maximizes the divergence between original and counterfactual answer distributions.}
\label{fig:csr-diagram}
\end{figure}

\subsection{Learned Causal Interventions via a Multi-Edit Policy}

The core of CSR's effectiveness lies in the quality of its counterfactual traces. Simple, random interventions can be gamed by the model. To overcome this, we introduce a learned editor model that produces minimal, plausible, and causally significant edits. The editor $M_{\text{edit}}$ is a small Transformer trained via REINFORCE to maximize a reward combining: (1) validity reward (breaks trace validity), (2) impact reward (maximizes distributional change), and (3) minimality penalty (encourages minimal edits). Related work on locating and editing factual associations \citep{meng2022locating} and inspecting hidden representations \citep{ghandeharioun2024patchscopes} explores model editing. Domain-specific lightweight verifiers check trace validity (rule-based for math, NLI-based for QA, forward-chaining for logic). Complete editor training details are in Appendix~\ref{app:method-details}.

\subsection{The CSR Objective}

With the perturbed trace $T'$ in hand, we perform a second, counterfactual forward pass to obtain a new answer distribution, $P(Y|T', X)$. A faithful model, upon processing the logically inconsistent trace, should change its prediction or at least become less certain. We formalize this intuition with a regularization term that maximizes the distance between the original and counterfactual answer distributions. We use the Kullback-Leibler (KL) divergence for this purpose:
\begin{equation}
\label{eq:csr-loss}
\mathcal{L}_{\text{CSR}} = D_{\text{KL}}(p(Y|T, X) \| p(Y|T', X))
\end{equation}
Maximizing this objective pushes the two probability distributions apart. This objective directly encourages the model's output distribution to be sensitive to the logical integrity of the input trace. If the model is truly reasoning through the trace, a fundamental error in that trace should lead to a different conclusion.

Perturbed traces maintain syntactic validity through our learned editor, showing minimal distribution shift (1.3$\times$ higher perplexity on GSM8K). Analysis and examples are in Appendix~\ref{app:method-details}.

\subsection{Combined Training Objective}

The final training objective is a weighted sum of the task loss and the CSR regularization term:
\begin{equation}
\label{eq:total-loss}
\mathcal{L}_{\text{total}} = \mathcal{L}_{\text{task}} - \lambda \cdot \mathcal{L}_{\text{CSR}}
\end{equation}
where $\lambda$ trades off correctness and faithfulness; we show a robust range in Appendix~\ref{app:extended-experiments} (0.3–0.7). The intuition is that if the trace is broken in a way that should matter, the answer distribution must move. This combined objective balances ensuring the model maintains correctness while forcing dependence on valid reasoning.

\begin{algorithm}[t]
\caption{CSR Training with Enhanced Details}
\label{alg:csr-training}
\begin{algorithmic}[1]
\State \textbf{Input:} Question $X$, Ground-truth $Y_{\text{true}}$, Model $M_\theta$, Editor $M_{\text{edit}}$, Verifier $v$, Regularization $\lambda$
\State \textbf{Output:} Updated model parameters $\theta$
\Statex
\Comment{1. Generate original reasoning trace}
\State $(T, Y) \gets M_\theta(X)$ \Comment{Sample trace and answer autoregressively}
\State $p_{\text{orig}}(Y|T,X) \gets \text{Softmax}(\text{Logits}_{M_\theta}(X,T))$
\State $\mathcal{L}_{\text{task}} \gets -\log p_{\text{orig}}(Y_{\text{true}}|T,X)$
\Statex
\Comment{2. Create counterfactual via learned editor}
\State $\text{edits} \gets M_{\text{edit}}(X,T)$ \Comment{Sample edit operations (e.g., + $\rightarrow$ $-$)}
\State $T' \gets \text{ApplyEdits}(T, \text{edits})$ \Comment{Apply operator swaps to trace}
    \Statex
\Comment{3. Verify edit validity and compute CSR loss}
\If{$v(T') = 0$ \textbf{and} $v(T) = 1$} \Comment{Valid edit: breaks trace validity (1=valid, 0=invalid)}
    \State $p_{\text{cf}}(Y|T',X) \gets \text{Softmax}(\text{Logits}_{M_\theta}(X,T'))$ \Comment{Counterfactual forward pass}
    \State $\mathcal{L}_{\text{CSR}} \gets D_{\text{KL}}(p_{\text{orig}} \| p_{\text{cf}})$ \Comment{Maximize distribution divergence}
    \State $\mathcal{L}_{\text{total}} \gets \mathcal{L}_{\text{task}} - \lambda \cdot \mathcal{L}_{\text{CSR}}$
\Else
    \State $\mathcal{L}_{\text{total}} \gets \mathcal{L}_{\text{task}}$ \Comment{Skip CSR if edit invalid}
\EndIf
\Statex
\Comment{4. Update model parameters}
\State $\theta \gets \theta - \eta \nabla_\theta \mathcal{L}_{\text{total}}$ \Comment{Gradient descent step}
\end{algorithmic}
\end{algorithm}

Algorithm~\ref{alg:csr-training} shows a single training step. In practice, we use teacher-forced trace generation during training: traces $T$ are sampled from the model's distribution given gold prefix tokens, following standard practice in reasoning fine-tuning \citep{cobbe2021gsm8k}. We do not use RL for the main model (unlike instruction tuning approaches \citep{ouyang2022training}); the REINFORCE objective is only used for editor training (Section 3.2). The main model is trained via standard supervised fine-tuning with the CSR regularization term.

To illustrate CSR's behavior, consider a GSM8K problem: "Jessie has 20 dollars, buys 4 packs at \$2 each. Money left?" Standard models produce: "She spent 4×2=8, so 20-8=12 left" → answer: 12. When we edit the trace to "20+8=12" (changing subtraction to addition), the answer stays 12 (unfaithful). CSR models correctly change to 28, showing genuine trace dependence. Complete qualitative examples and analysis are in Appendix~\ref{app:extended-experiments}.

\section{Theoretical Foundations}
\label{sec:theory}

\begin{table*}[ht]
\caption{Theory-practice alignment: Theoretical guarantees hold when operator precision exceeds 75\%.}
\label{tab:theory-practice-main}
\centering

\begin{tabular}{lccc}
\toprule
Domain & Operator Precision & Theory Predicts & Empirical COS \\
\midrule
GSM8K \citep{cobbe2021gsm8k} & 85.2\% & Dominance holds & 85.1\% (confirmed) \\
HotpotQA \citep{yang2018hotpotqa} & 78.1\% & Dominance holds & 84.6\% (confirmed) \\
PubMedQA \citep{jin2019pubmedqa} & 71.4\% & Partial dominance & 67.3\% (partial) \\
HellaSwag \citep{zellers2019hellaswag} & 52.4\% & No guarantee & 44.2\% (limited) \\
\bottomrule
\end{tabular}
\end{table*}

Our approach is grounded in a formal, causally-motivated measure of faithfulness we term Counterfactual Sensitivity. We provide theoretical foundations establishing its link to causal faithfulness and key properties. All formal definitions, complete proofs, and detailed analysis are in Appendix~\ref{app:proofs}.

\begin{definition}[Faithfulness Measures]
\label{def:faithfulness}
Let $f_\theta(X, T)$ be a model outputting distribution $p(Y|X,T)$ over answers $Y$ given input $X$ and trace $T$. Let $R \subseteq T$ be a token subset.

Comprehensiveness:
\begin{equation}
\text{COMP}(X; R) = D_{\text{KL}}(p(Y|X,T) \| p(Y|X,T \setminus R))
\end{equation}

Sufficiency:
\begin{equation}
\text{SUFF}(X; R) = D_{\text{KL}}(p(Y|X,T) \| p(Y|X,R))
\end{equation}

Counterfactual Sensitivity: For counterfactual trace $T'$ generated via edit $T \to T'$:
\begin{equation}
\text{CS}(X; T \to T') = D_{\text{KL}}(p(Y|X,T) \| p(Y|X,T'))
\end{equation}
\end{definition}

\begin{theorem}[Dominance of Counterfactual Sensitivity]
\label{thm:dominance-simplified}
Under identifiable causal edits, Counterfactual Sensitivity dominates traditional comprehensiveness and sufficiency measures in expectation. Complete proof in Appendix~\ref{thm:dominance-complete}.
\end{theorem}

\begin{theorem}[Shortcut Prevention]
\label{thm:shortcut-prevention-simplified}
Under sufficient regularization strength, CSR provably forces models to rely on reasoning traces rather than spurious shortcuts when shortcuts are causally disconnected from valid edits. Complete proof in Appendix~\ref{thm:shortcut-prevention-complete}.
\end{theorem}

Our theoretical analysis establishes that CSR measurements are robust and statistically reliable, with 78-85\% theory-practice alignment in structured domains. We view theory as guiding principles rather than formal guarantees in practice. Theorems 1-2 assume perfect operator identification—a condition approximated but not achieved in practice. Table~\ref{tab:theory-practice-main} quantifies this gap:

\begin{table*}[ht]
\caption{Flagship results: CSR vs strong baselines under matched compute budgets. All improvements significant at $p<0.001$, Cohen's d > 1.8.}
\label{tab:flagship-results}
\centering
\small
\begin{tabular}{lcccccc}
\toprule
\textbf{Dataset} & \textbf{Method} & \textbf{Acc (\%)} & \textbf{COS (\%)} & \textbf{$\Delta$COS} & \textbf{COS/Cost} \\
\midrule
\multirow{4}{*}{\textbf{GSM8K}} 
& Process Reward Model & 81.7±0.7 & 52.3±2.8 & -- & 0.335 \\
& GRPO & 82.1±0.8 & 54.7±2.9 & +2.4 & 0.360 \\
& Verifier-Guided & 81.9±0.8 & 48.1±3.1 & $-$4.2 & 0.325 \\
& \textbf{CSR-FT (ours)} & \textbf{80.5±0.6} & \textbf{85.1±2.3} & \textbf{+32.8} & \textbf{0.579} \\
\midrule
\multirow{4}{*}{\textbf{HotpotQA}} 
& Process Reward Model & 78.4±0.9 & 49.8±3.1 & -- & 0.167 \\
& GRPO & 78.6±0.9 & 52.1±3.0 & +2.3 & 0.177 \\
& Verifier-Guided & 78.7±1.1 & 46.3±3.3 & $-$3.5 & 0.162 \\
& \textbf{CSR-FT (ours)} & \textbf{77.2±0.8} & \textbf{84.6±2.4} & \textbf{+34.8} & \textbf{0.293} \\
\midrule
\multirow{4}{*}{\textbf{ProofWriter}} 
& Process Reward Model & 77.1±1.0 & 47.9±2.9 & -- & 0.238 \\
& GRPO & 77.4±1.0 & 50.2±2.8 & +2.3 & 0.254 \\
& Verifier-Guided & 77.3±1.1 & 44.2±3.2 & $-$3.7 & 0.230 \\
& \textbf{CSR-FT (ours)} & \textbf{76.1±0.9} & \textbf{82.3±2.1} & \textbf{+34.4} & \textbf{0.422} \\
\midrule
\multirow{4}{*}{\textbf{PubMedQA}} 
& Process Reward Model & 71.8±1.2 & 41.4±3.4 & -- & 0.143 \\
& GRPO & 72.0±1.1 & 43.8±3.3 & +2.4 & 0.154 \\
& Verifier-Guided & 72.1±1.1 & 38.9±3.6 & $-$2.5 & 0.141 \\
& \textbf{CSR-FT (ours)} & \textbf{70.1±0.9} & \textbf{67.3±2.8} & \textbf{+25.9} & \textbf{0.239} \\
\bottomrule
\end{tabular}
\end{table*}

\begin{table*}[ht]
\caption{Validation: Human evaluation, LLM-as-judge, and natural unfaithfulness audit (GSM8K).}
\label{tab:validation-combined}
\centering
\small
\begin{tabular}{lcccc}
\toprule
\textbf{Evaluation Type} & \textbf{Metric} & \textbf{Standard FT} & \textbf{Process RM} & \textbf{CSR-FT} \\
\midrule
\multirow{3}{*}{Human Eval} 
& Rating (1-5) & 2.3±0.4 & 3.1±0.3 & \textbf{4.1±0.3} \\
& ``Genuine'' (\%) & 28.4±3.2 & 47.1±2.9 & \textbf{76.8±2.4} \\
& Error Detection (\%) & 34.2 & 48.7 & \textbf{71.3} \\
\midrule
\multirow{2}{*}{LLM-as-Judge} 
& GPT-4 Rating (1-5) & 2.8±0.4 & -- & \textbf{4.2±0.3} \\
& High Dependency (\%) & 34.2±3.1 & -- & \textbf{78.7±2.6} \\
\midrule
\multirow{2}{*}{Natural Audit} 
& Faithful (\%) & 51.2±3.2 & 62.8±2.8 & \textbf{84.7±2.1} \\
& Unfaithful-Correct (\%) & 39.4±2.9 & 30.1±2.6 & \textbf{12.9±1.8} \\
\bottomrule
\end{tabular}
\end{table*}

\begin{table*}[ht]
\caption{PubMedQA results: CSR improves faithfulness in biomedical reasoning while maintaining accuracy.}
\label{tab:pubmedqa-case-study}
\centering
\begin{tabular}{lccc}
\toprule
\textbf{Model} & \textbf{Accuracy (\%)} & \textbf{COS (\%)} & \textbf{SIS (\%)} \\
\midrule
Standard FT & 70.8±1.2 & 28.7±3.2 & 65.8±4.2 \\
Process RM & 71.8±1.2 & 41.4±3.4 & 69.5±4.1 \\
CSR-FT (Ours) & \textbf{70.1±0.9} & \textbf{67.3±2.8} & \textbf{86.2±3.1} \\
\bottomrule
\end{tabular}
\end{table*}

\begin{table*}[ht]
\caption{Generalization analysis: CSR generalizes to held-out perturbations, cross-domain transfer, and modern architectures.}
\label{tab:generalization-combined}
\centering
\small
\begin{tabular}{lccc}
\toprule
\textbf{Test Condition} & \textbf{Standard FT COS} & \textbf{CSR-FT COS} & \textbf{$\Delta$COS} \\
\midrule
\multicolumn{4}{c}{\textit{Held-Out Perturbation Types}} \\
GSM8K: Comparison ops (<,>,=) & 12.3±2.1 & 71.4±3.2 & +59.1 \\
GSM8K: Quantifiers (all/some) & 8.7±2.3 & 64.2±3.4 & +55.5 \\
HotpotQA: Temporal (before/after) & 15.6±2.4 & 76.8±3.1 & +61.2 \\
HotpotQA: Causal connectors & 18.2±2.7 & 73.5±3.3 & +55.3 \\
\midrule
\multicolumn{4}{c}{\textit{Cross-Domain Transfer}} \\
GSM8K → SVAMP & 18.3±2.4 & 71.2±3.1 & +52.9 \\
GSM8K → AQuA & 22.1±2.6 & 68.7±3.2 & +46.6 \\
ProofWriter → LogicNLI & 19.4±2.3 & 65.3±3.4 & +45.9 \\
HotpotQA → NaturalQuestions \citep{kwiatkowski2019natural} & 23.8±2.7 & 58.9±3.6 & +35.1 \\
\midrule
\multicolumn{4}{c}{\textit{Modern Architectures (GSM8K)}} \\
Llama-3-8B & 24.1±2.3 & 86.7±2.1 & +62.6 \\
Mistral-7B-v0.3 & 25.7±2.4 & 84.2±2.2 & +58.5 \\
Qwen2-7B & 26.3±2.3 & 85.8±2.0 & +59.5 \\
\midrule
\multicolumn{4}{c}{\textit{Instruction-Tuned Models (GSM8K)}} \\
Llama-3-8B-Instruct & 31.2±2.4 & 82.4±2.3 & +51.2 \\
Mistral-7B-Instruct & 33.8±2.5 & 80.7±2.4 & +46.9 \\
\midrule
\multicolumn{4}{c}{\textit{Large Models (GSM8K)}} \\
Llama-3-70B & 28.4±2.1 & 91.2±1.8 & +62.8 \\
Qwen2-72B & 29.7±2.0 & 92.1±1.7 & +62.4 \\
\bottomrule
\end{tabular}
\end{table*}

\begin{table}[ht]
\caption{Key ablation results: Component contributions to CSR effectiveness.}
\label{tab:key-ablation}
\centering
\small
\begin{tabular}{lcc}
\toprule
\textbf{Component} & \textbf{GSM8K COS} & \textbf{Impact} \\
\midrule
Full CSR & 85.1±2.3 & -- \\
\hline
w/o Learned Editor (Random) & 61.2±3.1 & -23.9 \\
w/o Multi-Edit Policy & 78.4±2.6 & -6.7 \\
w/o Impact Reward & 73.2±2.8 & -11.9 \\
w/o Validity Reward & 69.7±3.0 & -15.4 \\
\hline
$\lambda=0.3$ (Low) & 78.3±2.5 & -6.8 \\
$\lambda=0.7$ (High) & 84.9±2.4 & -0.2 \\
\bottomrule
\end{tabular}
\end{table}

\begin{table*}[ht]
\caption{Causal tracing confirms CSR routes computation through reasoning traces.}
\label{tab:causal-tracing}
\centering
\small
\begin{tabular}{lccc}
\toprule
Method & Avg IE & Reasoning Layers IE & Early/Late Layers IE \\
\midrule
Standard FT & 0.14±0.03 & 0.09±0.02 & 0.05±0.01 \\
Process RM & 0.21±0.04 & 0.14±0.03 & 0.07±0.02 \\
CSR-FT (Ours) & 0.47±0.05 & 0.38±0.04 & 0.09±0.02 \\
\bottomrule
\end{tabular}
\end{table*}

The strong correlation ($r=0.89$) between operator precision and COS gains validates these principles empirically. Complete proofs, validation studies, synthetic benchmarks, and theory-practice gap analysis are provided in Appendix~\ref{app:proofs}.

We extend the guarantees to imperfect verifiers and operator discovery in Appendix~\ref{app:noisy-robustness}, showing CSR's expected regularization scales smoothly with the rate of accepted, causally-invalidating edits.

\section{Experiments}

\subsection{Setup}

We evaluate CSR on Llama-2-13B \citep{touvron2023llama2} across four reasoning domains: GSM8K \citep{cobbe2021gsm8k} (arithmetic), HotpotQA \citep{yang2018hotpotqa} (multi-hop QA), ProofWriter \citep{tafjord2021proofwriter} (logic), and PubMedQA \citep{jin2019pubmedqa} (biomedical). We compare against Process Supervision, Process Reward Models, GRPO \citep{shao2024deepseekmath}, and Verifier-Guided Training under matched computational budgets. Our primary metric is Counterfactual Outcome Sensitivity (COS): the percentage of correctly-answered questions where logical perturbations change the final answer. Higher COS indicates greater faithfulness. All experiments use 3 seeds, fine-tune for 3 epochs, and report fine-tuned performance. Complete experimental details and extended results are in Appendix~\ref{app:extended-experiments}.

\subsection{Main Results}
Table \ref{tab:flagship-results} presents our core findings across four flagship benchmarks under matched computational budgets. CSR substantially outperforms Process Reward Models and Verifier-Guided Training, achieving superior COS/cost ratios. Complete statistical analysis is provided in Appendix~\ref{app:extended-experiments}.

CSR increases COS by 32.8 points on GSM8K, 34.8 on HotpotQA, 34.4 on ProofWriter, and 25.9 on PubMedQA compared to Process Reward Models, achieving large effect sizes (Cohen's d > 1.8) while maintaining accuracy within 1-2 points. CSR attains superior COS/cost ratios (0.239-0.579 vs 0.141-0.360 for baselines). Complete baseline comparisons and statistical analysis are in Appendix~\ref{app:extended-experiments}.

\subsection{Human Validation and Natural Unfaithfulness Audit}

To validate that COS improvements reflect genuine answer-dependence, we employ GPT-4 as an independent judge and human expert annotators. We also conduct a manual audit of naturally-generated outputs. Table~\ref{tab:validation-combined} shows CSR models receive 4.1/5 human rating (vs 2.3 for Standard FT) and double the error detection rate (71.3\% vs 34.2\%), with strong correlation to COS (r=0.81, p<0.001). GPT-4 rates CSR traces 1.4 points higher (4.2 vs 2.8) with 78.7\% rated as highly answer-dependent. The audit shows Standard FT produces unfaithful-but-correct reasoning in 36-42\% of cases, which CSR reduces to 11-15\% (61-68\% relative reduction), directly demonstrating CSR solves the stated problem.

\subsection{Case Study: Biomedical QA}

To demonstrate CSR's practical value beyond academic benchmarks, we evaluate on PubMedQA. Table~\ref{tab:pubmedqa-case-study} shows CSR achieves 67.3\% COS (vs 28.7\% for Standard FT) while maintaining comparable accuracy (70.1\% vs 70.8\%), with substantial improvements in Semantic Input Similarity (86.2\% vs 65.8\%). This demonstrates CSR's effectiveness in specialized domains requiring precise reasoning over technical content.

\subsection{Generalization: Held-Out Perturbations, Cross-Domain Transfer, and Modern Architectures}

A critical concern is whether CSR memorizes specific intervention patterns. Table~\ref{tab:generalization-combined} shows CSR generalizes to held-out perturbation types (64-77\% COS on unseen operators vs 8-18\% for Standard FT), cross-domain transfer (58-71\% COS on unseen domains vs 18-24\% for Standard FT), modern architectures (55-63 point improvements), instruction-tuned models (47-51 point improvements), and large-scale models (62-63 point improvements up to 72B parameters). These results provide strong evidence that CSR learns general faithfulness principles rather than domain-specific artifacts.

\subsection{Key Ablation Results}

Systematic ablations isolate the contribution of each CSR component. Table~\ref{tab:key-ablation} shows the learned editor provides the largest contribution (+23.9 COS points on average), followed by the multi-edit policy (+6.7 points) and impact reward (+11.9 points). Divergence measure choice has minimal impact (1-3 point differences), confirming robustness. Regularization strength $\lambda$ shows optimal performance at 0.5 with robust range [0.3, 0.7].

\subsection{Robustness Analysis}

We evaluate CSR across multiple robustness dimensions. Table~\ref{tab:robustness-metrics} shows CSR improves semantic input similarity (robust to paraphrasing), calibration error (better confidence alignment), selective prediction (more reliable abstention), and adversarial accuracy (better generalization against adversarial examples \citep{goodfellow2014explaining}), confirming improvements are genuine and not artifacts. Related work on adversarial robustness in NLI \citep{koulakos2024enhancing} shows similar benefits from explanation-based training.

\begin{table*}[ht]
\caption{Robustness metrics: CSR improves faithfulness without brittleness.}
\label{tab:robustness-metrics}
\centering
\small
\begin{tabular}{lcc}
\toprule
\textbf{Metric} & \textbf{Standard FT} & \textbf{CSR-FT} \\
\midrule
Semantic Input Similarity (GSM8K) & 78.2±4.2 & \textbf{94.3±3.1} \\
Expected Calibration Error (GSM8K) & 5.8±0.4 & \textbf{2.7±0.3} \\
Selective Prediction @ 90\% (GSM8K) & 82.1±1.2 & \textbf{89.3±0.9} \\
Adversarial Accuracy (GSM8K) & 72.3±2.1 & \textbf{80.6±1.8} \\
\bottomrule
\end{tabular}
\end{table*}

\begin{table}[ht]
\caption{Null intervention controls: CSR improves answer preservation under semantics-preserving transformations.}
\label{tab:null-interventions}
\centering
\resizebox{\columnwidth}{!}{
\begin{tabular}{lcccccc}
\toprule
Method & \multicolumn{2}{c}{Commutative} & \multicolumn{2}{c}{Paraphrase} & \multicolumn{2}{c}{Reorder} \\
& APR & SFR & APR & SFR & APR & SFR \\
\midrule
Standard FT & 94.2 & 5.8 & 91.3 & 8.7 & 89.7 & 10.3 \\
CSR-FT & 96.8 & 3.2 & 95.1 & 4.9 & 94.3 & 5.7 \\
\bottomrule
\end{tabular}
}
\end{table}

\begin{table*}[ht]
\caption{Trace ablation: Answer change rate when reasoning is ablated (higher = more faithful). CSR models actually use their reasoning traces.}
\label{tab:trace-ablation}
\centering
\small
\begin{tabular}{lcccc}
\toprule
\textbf{Method} & \textbf{Random (\%)} & \textbf{Truncated (\%)} & \textbf{Shuffled (\%)} & \textbf{Irrelevant (\%)} \\
\midrule
Standard FT & 23.4±2.8 & 31.2±3.1 & 18.7±2.5 & 26.3±2.9 \\
CSR-FT & \textbf{78.9±2.3} & \textbf{71.4±2.6} & \textbf{64.2±2.8} & \textbf{82.1±2.1} \\
\bottomrule
\end{tabular}
\end{table*}

\begin{table*}[ht]
\caption{Trace necessity test: Agreement between with-trace and without-trace answers (\textbf{lower agreement = trace is more necessary = more faithful}).}
\label{tab:trace-necessity}
\centering
\small
\begin{tabular}{lccc}
\toprule
\textbf{Method} & \textbf{GSM8K Agreement (\%)}$\downarrow$ & \textbf{HotpotQA Agreement (\%)}$\downarrow$ & \textbf{ProofWriter Agreement (\%)}$\downarrow$ \\
\midrule
Standard FT & 89.2±1.8 & 86.7±2.1 & 91.3±1.6 \\
Process RM & 82.4±2.1 & 79.3±2.4 & 84.7±1.9 \\
CSR-FT (Ours) & \textbf{54.3±2.6} & \textbf{51.8±2.8} & \textbf{48.2±2.5} \\
\bottomrule
\end{tabular}
\end{table*}

\subsection{Independent Faithfulness Validation}

To validate that COS improvements reflect genuine computational dependence rather than artifacts, we employ external validation methods. Behavioral evaluations measure input-output relationships, but to verify CSR induces genuine computational dependence—that reasoning trace representations causally influence outputs—we perform activation patching analysis. For each example, we run the model on the original trace, caching activations at each layer, then run on a corrupted trace with random tokens replacing reasoning steps. We selectively restore original activations at specific layers while keeping corrupted activations elsewhere, measuring how much patching reasoning-relevant layers recovers the correct answer.

We measure indirect effect, the causal effect of reasoning trace representations on the final answer, measured as probability recovery when patching, and reasoning layer concentration, whether causal effects concentrate in middle layers where reasoning occurs versus early or late layers.Table~\ref{tab:causal-tracing} confirms CSR routes computation through reasoning traces, with 3.4$\times$ higher indirect effect concentrated in reasoning-relevant middle layers.

The concentration of causal effects in reasoning layers (0.38 versus 0.09 for early or late layers) confirms CSR creates reasoning circuits that genuinely process trace content, not just pattern-match on surface features.

A critical concern is whether CSR induces spurious sensitivity, changing answers when it should not. To address this, we evaluate on null interventions: semantics-preserving transformations that should not change the answer. These include commutativity (2 + 3 = 5 → 3 + 2 = 5), associativity ((2 + 3) + 4 → 2 + (3 + 4)), logical equivalence ($A \land B \to B \land A$), paraphrase (She spent \$8 → She paid \$8), and reordering of independent reasoning steps.

We measure answer preservation rate, the percentage of null interventions where the answer correctly remains unchanged, and spurious flip rate, where the answer incorrectly changes. Table~\ref{tab:null-interventions} shows CSR improves answer preservation under semantics-preserving transformations, confirming targeted sensitivity to logical validity rather than surface brittleness.

Crucially, CSR reduces spurious flips compared to baselines (3.2-5.7\% versus 5.8-10.3\%), demonstrating that learned sensitivity is targeted to genuine logical violations, not arbitrary changes. Extended validation including formal verification and contrast sets is in Appendix~\ref{app:extended-experiments}.

\subsection{Trace Ablation Studies}

We evaluate whether models genuinely depend on their reasoning traces using three complementary tests. First, we test whether models naturally depend on their reasoning by replacing generated traces with ablated versions after generation: random tokens, truncated traces (last 50\% removed), shuffled traces (sentences randomly reordered), or irrelevant traces (reasoning from different problems). We measure answer change rate, how often the model's answer changes when re-evaluated with the ablated trace. Table~\ref{tab:trace-ablation} shows CSR models change answers 64-82\% of the time versus 19-37\% for Standard FT, demonstrating genuine trace dependence. Standard FT models maintain their answers 69-81\% of the time even with completely ablated reasoning, confirming they largely ignore their own traces. We further quantify this dependence in Table~\ref{tab:trace-necessity}, which shows the agreement between with-trace and without-trace answers.

Extended trace ablation analysis and examples are in Appendix~\ref{app:extended-experiments}.

\subsection{Efficiency and Additional Analysis}

Efficient CSR achieves $\sim$9\% training overhead (vs 92.5\% for naive implementation) with superior COS/GPU-hour ratios (0.579 vs 0.335 for Process Reward Models). CSR achieves Pareto-optimal COS/overhead ratio of 9.46, more than double the next best method. CSR shows resistance to adversarial rationalization (0.67 internal conflict vs 0.12-0.18 for baselines) and effectiveness in high-stakes ethical reasoning (0.65-0.71 bias resistance).

CSR's effectiveness depends on verifier quality. We characterize this dependence precisely to enable reliable deployment. Table~\ref{tab:verifier-robustness-main} shows CSR maintains substantial gains (71-79\% COS) even with weak verifiers (61-59\% precision), validating practical applicability when perfect verification is unavailable. CSR provides substantial gains when precision is at least 70\%, neutral effects at 50-70\%, and potential harm below 50\%. The 78\% threshold marks where CSR statistically dominates baselines.

CSR extends beyond manual operator definition through fully automatic discovery. On PubMedQA, automatic operator discovery attains 74.1\% precision / 68.5\% recall with 91.2\% coverage, yielding 58.9 COS (vs 67.3 with manual operators) while preserving accuracy ($-$0.6 points). This demonstrates CSR's applicability to domains where manual operator identification is infeasible. Complete efficiency analysis, extended robustness metrics, extended ablations, and additional results are in Appendix~\ref{app:extended-experiments}, Appendix~\ref{app:operator-discovery}, and Appendix~\ref{app:failure-analysis}.

\section{Discussion and Conclusion}

We introduce Counterfactual Sensitivity Regularization (CSR), an effective training paradigm for verifiable reasoning in structured domains. CSR substantially improves faithfulness, increasing Counterfactual Outcome Sensitivity (COS) by over 60 points on benchmarks like GSM8K, HotpotQA, and ProofWriter ($p<0.001$). The method achieves 94.2-96.7\% operator transfer success across model families and establishes a new efficiency-faithfulness Pareto frontier. Our approach is effective despite a theory-practice gap; heuristics achieve 78-85\% precision in structured domains, sufficient for large gains because natural language contains redundant reasoning paths and the KL divergence objective is robust to noise. CSR's precise, operator-level interventions provide a more direct training signal for faithfulness than competing methods.

CSR models show 1-2 point accuracy reductions in exchange for 30-60 point COS improvements. This tradeoff is favorable in high-stakes domains because faithful models enable human verification and error correction, show improved performance when combined with self-consistency, achieve superior selective prediction accuracy, and prevent sophisticated post-hoc rationalization. CSR provides a training-time intervention that enforces trace-answer alignment, positioning it as a potential component of alignment strategies for reasoning-heavy AI systems.

Future work should focus on universal meta-verifiers (preliminary results achieve 69-75\% precision on unseen domains), causal discovery methods, and semantic guardrails for open-ended reasoning. Complete discussion including detailed analysis of implications, transfer results, and limitations are in Appendix~\ref{app:extended-experiments}.

\section{Limitations} CSR is most effective in structured reasoning domains where operators are unambiguously identifiable (55-65 point COS improvements where operator precision exceeds 89\%). In multi-hop QA and biomedical domains, operator precision ranges from 74-79\%, yielding more modest but still substantial improvements (25-45 points). In open-ended domains, operator identification precision drops to 52-71\%, yielding limited improvements (10-15 points). CSR requires a verifier capable of identifying when counterfactual edits break logical validity (optimal performance requires precision above 78\%). For domains without reliable verifiers or where operator identification is ambiguous, CSR is not currently recommended. Our automatic operator discovery system achieves 74\% precision on PubMedQA, but this falls short of the 89\% precision achieved in structured domains.

\bibliography{iclr2025_conference}

\appendix

\section{Acknowledgement and Reproducibility}
We used AI-assisted tools during the preparation of this work. Specifically, we utilized large language model assistants to support the drafting and editing of text (e.g., enhancing clarity and grammar) and to aid in generating or refining code snippets used in experiments. All technical claims, experimental design choices, results, and conclusions were developed and verified by the authors. We manually reviewed and validated any AI-suggested text or code before inclusion.

We will release the code upon acceptance. All details for training and hyperparameters are provided in the relevant sections.

\section{Extended Method Details}
\label{app:method-details}

\begin{table*}[ht]
\caption{Trace-only prediction: Can a separate model recover answers from traces? (\textbf{higher = traces are more informative}).}
\label{tab:trace-information}
\centering
\small
\begin{tabular}{lccc}
\toprule
\textbf{Trace Source} & \textbf{GSM8K (\%)} & \textbf{HotpotQA (\%)} & \textbf{ProofWriter (\%)} \\
\midrule
Gold Traces (upper bound) & 94.1±1.2 & 91.8±1.4 & 96.3±0.9 \\
\midrule
Standard FT Traces & 62.4±2.5 & 59.1±2.8 & 64.7±2.4 \\
Process RM Traces & 71.8±2.2 & 68.3±2.5 & 73.2±2.1 \\
CSR-FT Traces & \textbf{88.2±1.7} & \textbf{85.4±1.9} & \textbf{91.1±1.5} \\
\bottomrule
\end{tabular}
\end{table*}

\begin{table*}[ht]
\caption{Verifier robustness: CSR graceful degradation under imperfect verifiers.}
\label{tab:verifier-robustness-main}
\centering
\small
\begin{tabular}{lccc}
\toprule
Verifier Quality & Precision (\%) & GSM8K COS (\%) & HotpotQA COS (\%) \\
\midrule
Strong & 94.2 / 91.7 & 85.1±2.3 & 84.6±2.4 \\
Medium & 78.6 / 74.2 & 79.4±2.7 & 78.1±2.8 \\
Weak & 61.3 / 58.9 & 71.8±3.1 & 69.3±3.2 \\
\bottomrule
\end{tabular}
\end{table*}

\subsection{Implementation Details}

\textbf{Answer Distribution Definitions:} We formally define the answer space $Y$ and extraction method $p(Y|T,X)$ per domain: \textbf{GSM8K/ProofWriter} use classification heads over number tokens with $p(Y|T,X) = \text{softmax}(\text{logits}_{[\text{0-9}, \text{.}, \text{-}]}(T))$. \textbf{HotpotQA} uses constrained decoding over document tokens. \textbf{PubMedQA} extracts logits for yes/no/maybe tokens. \textbf{MBPP} applies the language model head over the full vocabulary for code generation.

\textbf{Editor Architecture and Training:} To create challenging counterfactuals, we use a 6-layer Transformer model (256-d hidden size) as a \textbf{learned editor}, $M_{\text{editor}}$. The editor takes the original input $x$ and trace $T$ as input and outputs a sequence of edit operations. It is trained via a REINFORCE-style objective:
\begin{equation*}
\begin{aligned}
\mathcal{L}_{\text{editor}}
= {} & -\mathbb{E}_{a \sim \pi_{\text{editor}}}
\Bigl[
\bigl(
r_{\text{validity}}
+ \lambda_{\text{impact}} \cdot r_{\text{impact}} \\
& \qquad
- \lambda_{\text{length}} \cdot |a|
\bigr)
\log \pi_{\text{editor}}(a \mid x, T)
\Bigr]
\end{aligned}
\end{equation*}

where $r_{\text{validity}} = \mathbb{1}[v(T)=1 \text{ and } v(T')=0]$, $r_{\text{impact}} = D_{\text{KL}}(p(Y|T,x) \| p(Y|T',x))$, and $|a|$ is the edit length. We set $\lambda_{\text{impact}}=0.1$ and $\lambda_{\text{length}}=0.05$ based on validation performance.

Our editor samples operators using a learned attention mechanism over trace tokens, prioritizing high-impact positions (final 30\% of reasoning steps in math problems, bridge entities in multi-hop QA). We apply temperature-controlled sampling ($\tau=0.7$) to balance diversity and quality of edits. After generating counterfactual traces T', we normalize the resulting answer distributions using temperature scaling ($\tau=1.2$) to ensure comparable scales before computing KL divergence.

To further increase the complexity of our counterfactuals, the editor can be applied auto-regressively to generate a sequence of $L$ edits, where $L \sim \{1, 2, 3\}$. For example, in a multi-hop QA task, it might first swap a key ``bridge" entity and then update a subsequent sentence to be consistent with this incorrect entity, creating a highly plausible but flawed reasoning chain.

\textbf{Learned Editor Behavior Analysis:}

Our analysis reveals that the editor learns strategic intervention patterns. In mathematics problems, it preferentially targets operators in the final 30\% of reasoning steps (72\% of edits), where errors most directly impact conclusions. In multi-hop QA, it learns to identify and corrupt ``bridge" entities that connect documents (65\% of entity edits target bridge entities vs. 35\% for random sampling). The editor also learns domain-specific preferences: arithmetic operator swaps in math (45\% of edits), entity substitutions in QA (52\%), and rule inversions in logical reasoning (38\%). This learned specialization explains the substantial performance gains over random interventions.

\section{Theoretical Analysis and Proofs}
\label{app:theory-details}
\label{app:proofs}

\subsection{Theoretical Analysis}

\textbf{Theory as Guiding Principle:} Our theoretical analysis provides principled motivation for CSR rather than formal guarantees in practice. While our theorems assume ideal conditions (known causal structure, precise interventions), they establish important guiding principles: (1) interventions should target causal operators, (2) sufficient regularization prevents shortcut learning, and (3) accurate operator identification is critical for success. Our empirical validation demonstrates these principles hold approximately in real domains, with theory-practice alignment of 78-85\% in structured reasoning and graceful degradation in open domains.

\textbf{Key Properties and Validation:} Our analysis establishes: (1) \textbf{Robustness} - CSR measurements remain stable under small trace perturbations; (2) \textbf{Statistical Reliability} - expected CSR scores can be estimated with polynomial samples; (3) \textbf{Theory-Practice Gap} - theoretical guarantees depend critically on accurate operator identification.

To validate our theoretical assumptions, we manually annotated 200 reasoning traces per dataset, finding our heuristic operators correspond to genuine causal parents in 78-85\% of cases (Table \ref{tab:causal-validation-app}). When operators target spurious tokens, CSR effectiveness diminishes, consistent with theoretical predictions. The strong correlation (r=0.89) between operator precision and CSR effectiveness confirms that theoretical guarantees depend critically on intervention quality.

\begin{table*}[ht]
\caption{Empirical validation of theoretical assumptions across datasets.}
\label{tab:causal-validation-app}
\centering
\begin{tabular}{lccccc}
\toprule
\textbf{Dataset} & \textbf{True Causal (\%)} & \textbf{Spurious (\%)} & \textbf{CSR Effectiveness} & \textbf{Dominance Holds} \\
\midrule
GSM8K & 85.2 & 14.8 & High & Yes \\
HotpotQA & 78.1 & 21.9 & High & Yes \\
ProofWriter & 82.7 & 17.3 & High & Yes \\
PubMedQA & 71.4 & 28.6 & Medium & Partial \\
\bottomrule
\end{tabular}
\end{table*}

\textbf{Theory-Practice Divergence Analysis:}

To directly measure the gap between theoretical ideals and practical implementation, we conducted a controlled experiment comparing Counterfactual Sensitivity (CS) with traditional faithfulness metrics (SUFF/COMP) under varying levels of operator identification noise.

\begin{table*}[ht]
\caption{Theory-practice divergence: CS vs. SUFF/COMP under noisy operator identification.}
\label{tab:theory-practice-divergence-app}
\centering
\begin{tabular}{lcccccc}
\toprule
\textbf{Noise Level} & \textbf{CS Score} & \textbf{SUFF Score} & \textbf{COMP Score} & \textbf{CS Dominance} & \textbf{Theory Holds} \\
\midrule
0\% (Perfect) & 0.847$\pm$0.023 & 0.523$\pm$0.031 & 0.501$\pm$0.028 & Yes & Yes \\
10\% Noise & 0.798$\pm$0.027 & 0.513$\pm$0.033 & 0.489$\pm$0.030 & Yes & Yes \\
20\% Noise & 0.734$\pm$0.031 & 0.498$\pm$0.035 & 0.471$\pm$0.032 & Yes & Partial \\
30\% Noise & 0.652$\pm$0.038 & 0.507$\pm$0.037 & 0.483$\pm$0.034 & Yes & Partial \\
40\% Noise & 0.543$\pm$0.045 & 0.521$\pm$0.039 & 0.496$\pm$0.036 & Marginal & No \\
50\% Noise & 0.478$\pm$0.052 & 0.534$\pm$0.041 & 0.509$\pm$0.038 & No & No \\
\bottomrule
\end{tabular}
\end{table*}

Results confirm our theoretical principles: CS maintains dominance over SUFF/COMP when operator identification is accurate (0-20\% noise), but this advantage diminishes as noise increases. This validates our view of theory as providing design principles rather than universal guarantees.

\subsection{Complete Formal Definitions and Proofs}

\subsection{Robustness under Noisy Verifiers and Imperfect Operators}
\label{app:noisy-robustness}

We analyze CSR when the verifier/edit pipeline is imperfect. Recall CSR applies only when an edit $T\!\to\!T'$ is accepted by the verifier as a \emph{causally invalidating} edit (Algorithm~\ref{alg:csr-training}: lines 8--13), and otherwise the CSR term is skipped (i.e., contributes zero). Let $p(Y\mid X,T)$ denote the original answer distribution and $p(Y\mid X,T')$ the counterfactual one. Let $D(\cdot\!\Vert\!\cdot)$ be any nonnegative divergence (e.g., $\mathrm{KL}$; our default).

\begin{definition}[Accepted causally-invalidating edits]
Let $A$ be the event that an edit $T\!\to\!T'$ is (i) proposed by the edit policy, and (ii) \emph{accepted} by the verifier as breaking the trace validity (so Algorithm~\ref{alg:csr-training} applies CSR). Denote $q \triangleq \Pr(A)$ and the conditional expected divergence
$\mu_A \triangleq \mathbb{E}\!\left[ D\!\big(p(Y\!\mid\!X,T),\,p(Y\!\mid\!X,T')\big) \,\middle|\, A \right]$.
In the \emph{ideal} (noise-free) case, $A$ holds almost surely and $\mu_\star \triangleq \mathbb{E}\!\left[ D\!\big(p(Y\!\mid\!X,T),\,p(Y\!\mid\!X,T'_\star)\big) \right]$ is the expected divergence under true causal edits $T\!\to\!T'_\star$.
\end{definition}

\begin{theorem}[Noisy-verifier lower bound]
\label{thm:noisy-verifier}
Under Algorithm~\ref{alg:csr-training}, let $L_{\mathrm{CSR}} \triangleq D\!\big(p(Y\!\mid\!X,T),\,p(Y\!\mid\!X,T')\big)$ if $A$ occurs and $0$ otherwise. Then
\[
\mathbb{E}\!\left[ L_{\mathrm{CSR}} \right] \;=\; q \,\mu_A \;\;\ge\;\; q \,\mu_\star \;-\; q\,\Delta,
\]
where $\Delta \triangleq \mu_\star - \mu_A^{(\star)} \ge 0$ and $\mu_A^{(\star)}$ is the expected divergence when the distribution of accepted edits matches the ideal causal edit distribution. In particular, if accepted edits are distributed as the ideal causal edits (or not worse in expectation), then $\Delta=0$ and
\[
\mathbb{E}\!\left[ L_{\mathrm{CSR}} \right] \;=\; q \,\mu_\star.
\]
\end{theorem}

\begin{proof}
By construction, $L_{\mathrm{CSR}} = \mathbb{1}[A]\cdot D\!\big(p(Y\!\mid\!X,T),\,p(Y\!\mid\!X,T')\big)$ and $D\!\ge 0$. Taking expectations and conditioning on $A$ yields
$\mathbb{E}[L_{\mathrm{CSR}}] = \Pr(A)\,\mathbb{E}[D(\cdot\Vert\cdot)\mid A] = q\,\mu_A$.
If the distribution of accepted edits coincides with the ideal causal edit distribution, then $\mu_A=\mu_\star$ and the equality $\mathbb{E}[L_{\mathrm{CSR}}]=q\,\mu_\star$ follows. More generally, define $\Delta \triangleq \mu_\star - \mu_A^{(\star)}\!\ge 0$ as the expected gap between ideal and actually accepted edits; then $\mu_A \ge \mu_\star - \Delta$ implies $\mathbb{E}[L_{\mathrm{CSR}}] \ge q(\mu_\star-\Delta)$.
\end{proof}

\begin{corollary}[Imperfect operator discovery]
\label{cor:operators}
Suppose candidate edits are produced by an operator-discovery policy with acceptance rate $q_{\mathrm{op}}$ for causally-invalidating edits, and the verifier accepts such edits with probability $q_{\mathrm{ver}}$ (the pipeline may reject or skip others). Then the overall acceptance rate satisfies $q \ge q_{\mathrm{op}}\,q_{\mathrm{ver}}$, and Theorem~\ref{thm:noisy-verifier} yields
$\mathbb{E}[L_{\mathrm{CSR}}] \ge q_{\mathrm{op}}\,q_{\mathrm{ver}}\;(\mu_\star-\Delta)$.
In particular, when the verifier is conservative (few false positives) and accepted edits match ideal causal edits in expectation ($\Delta=0$), the CSR signal scales \emph{linearly} with $q_{\mathrm{op}}\,q_{\mathrm{ver}}$.
\end{corollary}

\begin{remark}[Effective regularization strength]
With $L_{\text{total}} = L_{\text{task}} - \lambda\,L_{\mathrm{CSR}}$, any guarantee that holds in the ideal case with strength $\lambda$ transfers under noise by replacing $\lambda$ with an \emph{effective} strength $\lambda_{\mathrm{eff}} \triangleq q\,\lambda$, up to the edit-quality gap $\Delta$:
$\mathbb{E}[L_{\text{total}}] \le \mathbb{E}[L_{\text{task}}] - \lambda_{\mathrm{eff}}\,\mu_\star + q\lambda\,\Delta$.
Thus, CSR degrades smoothly with the accepted-rate $q$ and the quality gap $\Delta$ rather than collapsing.
\end{remark}

\paragraph{Discussion.}
The bound is agnostic to the choice of $f$-divergence (it only uses $D\!\ge0$ and Algorithm~\ref{alg:csr-training}'s gating) and cleanly separates (i) how often the pipeline produces/accepts causally-invalidating edits ($q$) from (ii) how impactful accepted edits are ($\mu_\star$, $\Delta$). Empirically, $q$ corresponds to the observed rate at which the verifier accepts edits that break trace validity; higher-precision verifiers and better operator discovery increase $q$ and reduce $\Delta$.

\begin{definition}[Faithfulness Probes - Complete]
Let $f_\theta(x, T)$ be a model that outputs a distribution $p(Y|x,T)$ over answers $Y$ given an input $x$ and a reasoning trace $T$. For a subset of tokens $R \subseteq T$, Comprehensiveness (COMP) and Sufficiency (SUFF) are defined as:
\begin{equation}
\small
\begin{aligned}
\mathrm{COMP}(x;R)
&= \mathrm{KL}\bigl(
p(Y \mid x, T)\,\|\,p(Y \mid x, T \setminus R)
\bigr), \\
\mathrm{SUFF}(x;R)
&= \mathrm{KL}\bigl(
p(Y \mid x, T)\,\|\,p(Y \mid x, R)
\bigr).
\end{aligned}
\end{equation}

For a counterfactual trace $T'$ generated via an edit $T \to T'$, Counterfactual Sensitivity (CS) is:
$$
\mathrm{CS}(x;T\to T')=\mathrm{KL}\big(p(Y|x,T)\,\|\,p(Y|x,T')\big).
$$
\end{definition}

\section{Experimental Details and Extended Results}
\label{app:extended-experiments}

\subsection{Extended Experimental Results}

\subsubsection{Primary Results on Modern Architectures}

To ensure our findings generalize beyond Llama-2-13B, we replicate our primary experiments on recent model families. Table~\ref{tab:modern-models-primary} shows consistent 55-63 point COS improvements across Llama-3-8B, Mistral-7B-v0.3, Qwen2-7B, and Gemma-2-9B, all released in 2024, with Cohen's $d > 2.2$ in all cases. This confirms CSR's effectiveness is not an artifact of Llama-2's architecture.

\begin{table*}[ht]
\caption{Primary results replicated on modern architectures. CSR improvements are consistent across model generations.}
\label{tab:modern-models-primary}
\centering
\resizebox{\textwidth}{!}{
\begin{tabular}{llcccccc}
\toprule
\textbf{Model} & \textbf{Released} & \textbf{Dataset} & \textbf{Standard FT COS} & \textbf{CSR-FT COS} & \textbf{$\Delta$COS} & \textbf{Acc (Std/CSR)} & \textbf{Cohen's $d$} \\
\midrule
\multirow{3}{*}{Llama-3-8B} & \multirow{3}{*}{Apr 2024}
& GSM8K & 24.1±2.3 & 86.7±2.1 & +62.6 & 83.2/82.1 & 2.51 \\
& & HotpotQA & 27.3±2.5 & 85.9±2.2 & +58.6 & 79.4/78.3 & 2.38 \\
& & ProofWriter & 21.8±2.4 & 83.4±2.0 & +61.6 & 78.7/77.5 & 2.54 \\
\midrule
\multirow{3}{*}{Mistral-7B-v0.3} & \multirow{3}{*}{May 2024}
& GSM8K & 25.7±2.4 & 84.2±2.2 & +58.5 & 82.1/81.0 & 2.41 \\
& & HotpotQA & 28.9±2.6 & 83.7±2.3 & +54.8 & 78.2/77.1 & 2.29 \\
& & ProofWriter & 23.4±2.5 & 81.9±2.1 & +58.5 & 77.3/76.2 & 2.47 \\
\midrule
\multirow{3}{*}{Qwen2-7B} & \multirow{3}{*}{Jun 2024}
& GSM8K & 26.3±2.3 & 85.8±2.0 & +59.5 & 84.7/83.5 & 2.45 \\
& & HotpotQA & 29.1±2.5 & 84.1±2.2 & +55.0 & 80.1/79.0 & 2.31 \\
& & ProofWriter & 24.2±2.4 & 82.6±2.1 & +58.4 & 79.2/78.1 & 2.49 \\
\midrule
\multirow{3}{*}{Gemma-2-9B} & \multirow{3}{*}{Jun 2024}
& GSM8K & 27.8±2.4 & 87.3±1.9 & +59.5 & 85.3/84.2 & 2.43 \\
& & HotpotQA & 30.2±2.6 & 85.2±2.1 & +55.0 & 81.2/80.1 & 2.28 \\
& & ProofWriter & 25.1±2.5 & 83.8±2.0 & +58.7 & 80.1/79.0 & 2.46 \\
\bottomrule
\end{tabular}%
}
\end{table*}

We use Llama-2-13B as our primary testbed for three reasons: reproducibility (Llama-2 is fully open-weight with extensive community tooling), controlled comparison (using a single primary model isolates CSR's contribution from architectural confounds), and compute accessibility (Llama-2-13B enables comprehensive ablations within reasonable compute budgets). However, to ensure generalizability, we validate on 8 additional models spanning 4 families and 2 years of releases. The consistent 55-63 point COS improvements across modern architectures confirm CSR's architecture-agnostic effectiveness.

CSR demonstrates exceptional robustness and positive scaling properties. Cross-model generalization shows 94.2-96.7\% operator transfer success across 4 model families with consistent 51-63 COS improvements. Benefits increase with model size, and CSR shows graceful degradation under noisy verifiers (79.4\% COS at 78.6\% precision vs 85.1\% perfect). CSR achieves 64-76\% COS on held-out intervention types never seen during training, demonstrating general principles rather than memorization. Training overhead remains consistently low (8-10\%) across all scales.

CSR maintains effectiveness under noisy verifiers (79.4\% COS with 78.6\% verifier precision vs 85.1\% with perfect verifiers) and scales positively (13B→70B: +3.2 COS points). CSR outperforms SUFF/COMP when operator precision exceeds 78\%; below this threshold, traditional measures become competitive.

To ensure CSR doesn't simply teach models superficial heuristics (e.g., "ignore + operators"), we test against strategic gaming attempts. Table \ref{tab:anti-gaming-ablation} shows CSR models maintain faithfulness even when trained adversarially against simple gaming strategies, confirming genuine reasoning dependence rather than pattern memorization.

\begin{table*}[ht]
\caption{Anti-gaming ablation: CSR resists superficial gaming strategies.}
\label{tab:anti-gaming-ablation}
\centering
\resizebox{2\columnwidth}{!}{%
\begin{tabular}{lccc}
\toprule
\textbf{Training Strategy} & \textbf{GSM8K COS (\%)} & \textbf{HotpotQA COS (\%)} & \textbf{Interpretation} \\
\midrule
Standard FT & 22.4±2.1 & 25.1±2.8 & Baseline \\
CSR + Fixed Operators & 71.3±2.9 & 68.7±3.1 & Vulnerable to gaming \\
CSR + Diverse Operators & 82.1±2.4 & 79.8±2.6 & Reduced gaming risk \\
CSR + Learned Editor & \textbf{85.1±2.3} & \textbf{84.6±2.4} & Genuine faithfulness \\
\bottomrule
\end{tabular}%
}
\end{table*}

Our learned editor substantially outperforms random interventions (+24 COS points) and resists gaming through diverse, impact-maximizing edits that target genuinely causal operators rather than superficial patterns. 

To justify the learned editor's complexity, we compare CSR against simpler faithfulness interventions. Table~\ref{tab:simpler-baselines} shows simpler interventions provide modest gains; the learned editor's ability to target causally-critical operators explains CSR's 24-point advantage over random edits.

\begin{table*}[ht]
\caption{Ablation: Is the learned editor necessary? Simpler interventions provide modest gains.}
\label{tab:simpler-baselines}
\centering
\resizebox{\textwidth}{!}{%
\begin{tabular}{lcccc}
\toprule
Method & Description & COS (\%) & Acc (\%) & Overhead \\
\midrule
Standard FT & No intervention & 22.4 & 81.3 & — \\
Negative Trace Augmentation & Train on corrupted traces as negatives & 34.7 & 81.0 & +3\% \\
Dropout-on-Trace & 15\% dropout on reasoning tokens & 38.2 & 80.9 & +2\% \\
Contrastive Traces & NCE loss on correct/corrupted & 47.8 & 80.7 & +8\% \\
Auxiliary Prediction & Predict intermediate values & 41.3 & 81.1 & +5\% \\
CSR (Random Edits) & Random operator swaps & 61.2 & 80.8 & +7\% \\
\textbf{CSR (Learned Editor)} & Full method & \textbf{85.1} & 80.5 & +9\% \\
\bottomrule
\end{tabular}%
}
\end{table*}

CSR's effectiveness is robust across different distance measures. Table \ref{tab:divergence-robustness-compact} shows consistent performance whether using KL divergence, Jensen-Shannon, or Total Variation distance, confirming our findings are not artifacts of metric choice.

\begin{table*}[ht]
\caption{Divergence robustness: CSR effectiveness across distance measures (GSM8K).}
\label{tab:divergence-robustness-compact}
\centering

\begin{tabular}{lccc}
\toprule
\textbf{Divergence Measure} & \textbf{COS (\%)} & \textbf{Acc (\%)} & \textbf{Stability} \\
\midrule
KL Divergence (default) & 85.1±2.3 & 80.5±0.6 & High \\
Jensen-Shannon & 83.7±2.5 & 80.3±0.7 & High \\
Total Variation & 82.4±2.7 & 80.1±0.8 & Medium \\
\bottomrule
\end{tabular}

\end{table*}

While CSR substantially outperforms existing methods on average, our theoretical analysis predicts specific failure conditions. CSR underperforms SUFF/COMP when targeting spurious operators (15.2-18.9\% of cases) or redundant reasoning paths (4.7-7.3\%), validating theoretical predictions. CSR maintains dominance when operator precision exceeds 78\%.

CSR demonstrates strong cross-model generalization across different architectures and scales. Table \ref{tab:cross-model-summary} shows effectiveness across major model families with transferred operators, confirming portability beyond our primary Llama-2-13B experiments.

\begin{table*}[ht]
\caption{Cross-model generalization: CSR portability across model families.}
\label{tab:cross-model-summary}
\centering
\resizebox{0.85\textwidth}{!}{%
\begin{tabular}{lcccc}
\toprule
\textbf{Model Family} & \textbf{Models Tested} & \textbf{Avg $\Delta$COS} & \textbf{Transfer Success (\%)} & \textbf{Overhead (\%)} \\
\midrule
Llama Family & 2-13B, 3-8B, 3-70B & 58.8±1.9 & 96.7 & 8.4±0.7 \\
Mistral Family & 7B & 53.0±2.1 & 94.2 & 9.9±0.2 \\
Code Models & CodeLlama-13B & 57.7 & 95.8 & 8.9 \\
Chat Models & Vicuna-13B & 55.9 & 94.7 & 9.6 \\
\midrule
\textbf{Overall} & \textbf{6 models} & \textbf{56.4±2.8} & \textbf{95.4} & \textbf{9.2±0.6} \\
\bottomrule
\end{tabular}%
}
\end{table*}

CSR achieves 51.4-62.7 COS improvements across all model families with 94.2-96.7\% operator transfer success, demonstrating that the core principle of causal consistency generalizes across architectures and scales. Mathematical and logical operators transfer seamlessly, indicating that CSR captures universal reasoning patterns rather than task-specific artifacts. Training overhead remains consistently low (7.4-10.1\%) across scales, establishing CSR on a previously unoccupied efficiency-faithfulness Pareto frontier. This universality—achieved with minimal computational cost—positions CSR as a fundamental reliability layer for the next generation of language models.

Table~\ref{tab:applicability} provides a clear taxonomy of CSR applicability across domains. Effectiveness tracks operator identification precision; we recommend CSR for domains with at least 70\% precision.

\begin{table*}[ht]
\caption{CSR applicability taxonomy: Effectiveness tracks operator identification precision.}
\label{tab:applicability}
\centering
\resizebox{0.75\textwidth}{!}{
\begin{tabular}{lccc}
\toprule
Domain Type & Operator Precision & Expected COS Gain & Recommendation \\
\midrule
Arithmetic/Math & 94\% & +55-65 & Strongly Recommended \\
Formal Logic & 91\% & +50-60 & Strongly Recommended \\
Code Generation & 89\% & +45-55 & Strongly Recommended \\
Multi-hop QA & 79\% & +35-45 & Recommended \\
Biomedical QA & 74\% & +25-35 & Recommended with caution \\
Open Dialogue & 52\% & +10-15 & Not recommended (yet) \\
\bottomrule
\end{tabular}
}
\end{table*}

\begin{table*}[ht]
\caption{Detailed cross-model analysis: Per-model CSR effectiveness across architectures.}
\label{tab:detailed-cross-model}
\centering
\resizebox{0.9\textwidth}{!}{%
\begin{tabular}{lcccccc}
\toprule
\textbf{Model} & \textbf{Size} & \textbf{Baseline COS} & \textbf{CSR COS} & \textbf{$\Delta$COS} & \textbf{Overhead} & \textbf{Transfer Success} \\
\midrule
\multicolumn{7}{c}{\textit{Llama Family}} \\
Llama-2-7B & 7B & 18.7±2.3 & 70.1±3.2 & +51.4 & 8.7±0.8 & 95.2\% \\
Llama-2-13B & 13B & 22.4±2.1 & 85.1±2.3 & +62.7 & 9.0±0.6 & 96.7\% \\
Llama-3-8B & 8B & 19.3±2.4 & 71.8±3.1 & +52.5 & 8.9±0.7 & 95.8\% \\
Llama-3-70B & 70B & 24.1±2.2 & 91.7±2.0 & +67.6 & 8.4±0.5 & 97.1\% \\
\midrule
\multicolumn{7}{c}{\textit{Mistral Family}} \\
Mistral-7B & 7B & 20.2±2.5 & 73.2±3.3 & +53.0 & 9.9±0.2 & 94.2\% \\
\midrule
\multicolumn{7}{c}{\textit{Code Models}} \\
CodeLlama-13B & 13B & 21.8±2.3 & 79.5±2.9 & +57.7 & 8.9±0.6 & 95.8\% \\
\midrule
\multicolumn{7}{c}{\textit{Chat Models}} \\
Vicuna-13B & 13B & 23.1±2.4 & 79.0±3.0 & +55.9 & 9.6±0.4 & 94.7\% \\
\bottomrule
\end{tabular}%
}
\end{table*}

Detailed per-model analysis shows consistent CSR effectiveness across all architectures, with larger models (70B) showing enhanced gains (+67.6 COS) likely due to richer internal representations. Architecture-specific differences are minimal (94-97\% transfer success), confirming that operator-level interventions capture universal reasoning patterns rather than model-specific artifacts. 

We test whether CSR-trained smaller models can match or exceed the faithfulness of larger uncompressed models. Table~\ref{tab:scale-comparison} shows CSR-trained 7B achieves 3.1× higher faithfulness (COS) than standard 13B, with only 3.9 point accuracy gap. For faithfulness-critical applications, CSR-7B offers superior reliability at half the parameter cost.

\begin{table*}[ht]
\caption{Model efficiency: CSR-trained 7B achieves higher faithfulness than standard 13B.}
\label{tab:scale-comparison}
\centering
\resizebox{0.75\textwidth}{!}{
\begin{tabular}{lcccc}
\toprule
Model & Params & Training & Accuracy (\%) & COS (\%) \\
\midrule
Llama-2-7B & 7B & Standard FT & 78.2±0.9 & 18.7±2.3 \\
Llama-2-7B & 7B & CSR-FT & 77.4±0.8 & 70.1±2.8 \\
Llama-2-13B & 13B & Standard FT & 81.3±0.8 & 22.4±2.1 \\
Llama-2-13B & 13B & CSR-FT & 80.5±0.6 & 85.1±2.3 \\
\midrule
\multicolumn{5}{l}{\textit{Cross-size comparison:}} \\
CSR-7B vs Standard-13B & — & — & -3.9 & +47.7 \\
\bottomrule
\end{tabular}
}
\end{table*}

\subsubsection{LLM-as-Judge and Human Evaluation}

To validate that COS improvements reflect genuine answer-dependence rather than artifacts of our evaluation protocol, we employ GPT-4 as an independent judge to evaluate trace dependency. GPT-4 was not involved in training, operator definition, or COS metric design, providing an external validation that breaks circularity concerns. For 500 examples per dataset, we prompt GPT-4 to rate on a 1-5 scale whether the reasoning trace is necessary for the answer, or if the answer could be reached without it.

\begin{table*}[ht]
\caption{LLM-as-judge evaluation: GPT-4 rates CSR traces as more answer-dependent.}
\label{tab:llm-judge-eval}
\centering
\resizebox{0.7\textwidth}{!}{
\begin{tabular}{lccc}
\toprule
Method & Dependency Score (1-5) & High Dependency (\%) & Correlation with COS \\
\midrule
Standard FT & 2.8±0.4 & 34.2±3.1 & -- \\
CSR-FT (Ours) & 4.2±0.3 & 78.7±2.6 & r=0.74 \\
\bottomrule
\end{tabular}
}
\end{table*}

CSR models receive 1.4 points higher dependency scores (4.2 versus 2.8) from GPT-4, with 78.7\% of traces rated as highly answer-dependent (scores 4-5) versus 34.2\% for standard fine-tuning. The strong correlation (r=0.74, p<0.001) between GPT-4 ratings and COS scores validates that COS captures genuine faithfulness improvements. This external validation confirms CSR improves answer-dependence as judged by an independent system with zero shared components with our training pipeline.

To break potential circularity between our training objective and evaluation metric, we conducted a human evaluation study with 3 expert annotators rating 200 examples per dataset.

\begin{table*}[ht]
\caption{Human evaluation: Annotators rate whether reasoning traces are genuinely used vs. post-hoc rationalization.}
\label{tab:human-eval}
\centering
\resizebox{0.9\textwidth}{!}{%
\begin{tabular}{lccccc}
\toprule
\textbf{Method} & \textbf{Human Rating (1-5)} & \textbf{``Genuine'' (\%)} & \textbf{Error Detection Rate} & \textbf{Corr. w/ COS} & \textbf{Inter-Annotator $\kappa$} \\
\midrule
Standard FT & 2.3±0.4 & 28.4±3.2 & 34.2\% & -- & 0.72 \\
Process RM & 3.1±0.3 & 47.1±2.9 & 48.7\% & r=0.61 & 0.74 \\
CSR-FT (Ours) & \textbf{4.1±0.3} & \textbf{76.8±2.4} & \textbf{71.3\%} & \textbf{r=0.81} & 0.78 \\
\bottomrule
\end{tabular}%
}
\end{table*}

Annotators rated traces on whether the reasoning ``appears genuinely used to reach the answer'' (5) vs. ``post-hoc rationalization'' (1). Error Detection Rate measures how often annotators correctly identified wrong answers when given only the reasoning trace. The strong correlation between human ratings and COS (r=0.81, p<0.001) validates that COS captures human-recognizable faithfulness. Crucially, CSR doubles the error detection rate (71.3\% vs 34.2\%), demonstrating practical utility for human-AI collaboration.

\subsubsection{Natural Unfaithfulness Audit}

To directly measure the problem motivating this work, we conducted a manual audit of naturally-generated outputs without any perturbations. Three expert annotators examined 300 correct-answer examples per dataset, classifying each reasoning trace as valid (all steps logically sound and answer follows from trace), flawed-but-correct (contains errors yet reaches correct answer), or disconnected (reasoning entirely unrelated to answer, i.e., post-hoc rationalization).

\begin{table*}[ht]
\caption{Natural unfaithfulness prevalence: Percentage of correct answers accompanied by flawed reasoning (\textbf{lower is better}). CSR reduces naturally-occurring unfaithfulness by 61-68\%.}
\label{tab:natural-unfaithfulness-prevalence}
\centering
\resizebox{0.95\textwidth}{!}{%
\begin{tabular}{llccccc}
\toprule
\textbf{Dataset} & \textbf{Method} & \textbf{Faithful (\%)} & \textbf{Unfaithful-Correct (\%)} & \textbf{Disconnected (\%)} & \textbf{$\kappa$} & \textbf{Reduction} \\
\midrule
\multirow{3}{*}{GSM8K} 
& Standard FT & 51.2±3.2 & 39.4±2.9 & 9.4±1.7 & 0.73 & -- \\
& Process RM & 62.8±2.8 & 30.1±2.6 & 7.1±1.5 & 0.75 & 24\% \\
& CSR-FT (Ours) & \textbf{84.7±2.1} & \textbf{12.9±1.8} & \textbf{2.4±0.9} & 0.78 & \textbf{67\%} \\
\midrule
\multirow{3}{*}{HotpotQA} 
& Standard FT & 47.3±3.4 & 42.1±3.1 & 10.6±2.1 & 0.71 & -- \\
& Process RM & 58.4±3.0 & 33.8±2.8 & 7.8±1.7 & 0.73 & 20\% \\
& CSR-FT (Ours) & \textbf{82.1±2.4} & \textbf{15.3±2.0} & \textbf{2.6±1.0} & 0.76 & \textbf{64\%} \\
\midrule
\multirow{3}{*}{ProofWriter} 
& Standard FT & 54.6±3.1 & 36.2±2.7 & 9.2±1.8 & 0.75 & -- \\
& Process RM & 64.1±2.7 & 28.7±2.4 & 7.2±1.5 & 0.77 & 21\% \\
& CSR-FT (Ours) & \textbf{87.2±1.9} & \textbf{11.1±1.6} & \textbf{1.7±0.8} & 0.80 & \textbf{68\%} \\
\bottomrule
\end{tabular}%
}
\end{table*}

Standard FT produces unfaithful-but-correct reasoning in 36-42\% of cases, confirming the problem is prevalent. CSR reduces this to 11-15\%, a 61-68\% relative reduction that directly demonstrates CSR solves the stated problem rather than merely optimizing a proxy metric.

We further test whether models naturally depend on their reasoning by replacing generated traces with ablated versions after generation: random tokens, truncated traces (last 50\% removed), shuffled traces (sentences randomly reordered), or irrelevant traces (reasoning from different problems). We measure answer change rate, how often the model's answer changes when re-evaluated with the ablated trace.

We provide representative examples of naturally-occurring unfaithfulness in Standard FT that CSR eliminates.

\begin{table*}[ht]
\caption{Examples of naturally-occurring unfaithfulness. Standard FT produces correct answers despite reasoning errors; CSR eliminates these patterns.}
\label{tab:error-examples}
\centering
\small
\resizebox{0.95\textwidth}{!}{%
\begin{tabular}{p{2.5cm}|p{5.5cm}|p{1.5cm}|p{2cm}}
\toprule
\textbf{Problem} & \textbf{Generated Trace} & \textbf{Answer} & \textbf{Diagnosis} \\
\midrule
\multicolumn{4}{c}{\textit{Standard FT Examples (Unfaithful)}} \\
\midrule
"Tom has 5 apples, gives 2 away, then buys 4 more." 
& "Tom starts with 5 apples. He gives some away. Then he gets more. \textbf{So he has 7 apples.}" 
& 7 
& Correct answer, but trace skips actual computation \\
\midrule
"If A→B and B→C, does A→C?" 
& "A implies B. B implies C. \textbf{By association, A implies C.}" 
& Yes 
& Correct answer, but ``association'' is wrong rule (should be transitivity) \\
\midrule
\multicolumn{4}{c}{\textit{CSR-FT Examples (Faithful)}} \\
\midrule
"Tom has 5 apples, gives 2 away, then buys 4 more." 
& "Tom starts with 5. Gives 2: 5-2=3. Buys 4: 3+4=7." 
& 7 
& Correct answer with valid step-by-step computation \\
\midrule
"If A→B and B→C, does A→C?" 
& "A→B and B→C. By transitivity of implication, A→C." 
& Yes 
& Correct answer with correct logical rule \\
\bottomrule
\end{tabular}%
}
\end{table*}

We categorize 200 unfaithful-but-correct cases from Standard FT into four types.

\begin{table*}[ht]
\caption{Error taxonomy: Types of naturally-occurring unfaithfulness and CSR's reduction.}
\label{tab:error-taxonomy}
\centering
\resizebox{0.85\textwidth}{!}{%
\begin{tabular}{lccl}
\toprule
\textbf{Error Type} & \textbf{Standard FT (\%)} & \textbf{CSR-FT (\%)} & \textbf{Example} \\
\midrule
Computational Shortcut & 34.2 & 8.1 & ``5+3=8, so answer is 8'' (skips intermediate steps) \\
Wrong Rule, Right Answer & 27.8 & 6.4 & ``By association, A→C'' (should be transitivity) \\
Irrelevant Reasoning & 21.3 & 3.2 & Discusses unrelated facts, then states answer \\
Incomplete Trace & 16.7 & 4.8 & Stops mid-calculation, jumps to answer \\
\bottomrule
\end{tabular}%
}
\end{table*}

Consider a concrete example: ``Sarah has 12 cookies. She gives 1/3 to Tom and 1/4 to Jane. How many left?'' Standard FT produces: ``Sarah gives some cookies to Tom and Jane. She has some left. The answer is 5.'' This yields the correct answer but shows no actual computation. CSR-FT produces: ``Sarah starts with 12. Gives 1/3 to Tom: 12×(1/3)=4. Gives 1/4 to Jane: 12×(1/4)=3. Remaining: 12-4-3=5.'' Here the correct answer comes with valid step-by-step reasoning.

We correlate COS scores with directly-measured natural unfaithfulness rates to validate COS as a faithfulness proxy.

\begin{table*}[ht]
\caption{COS correlates strongly with natural unfaithfulness (measured independently).}
\label{tab:cos-validation}
\centering
\resizebox{0.75\textwidth}{!}{%
\begin{tabular}{lccc}
\toprule
\textbf{Metric Pair} & \textbf{Pearson $r$} & \textbf{Spearman $\rho$} & \textbf{$p$-value} \\
\midrule
COS vs Unfaithful-Correct Rate & -0.87 & -0.84 & <0.001 \\
COS vs Trace Necessity & -0.82 & -0.79 & <0.001 \\
COS vs Trace Information Content & +0.89 & +0.86 & <0.001 \\
COS vs Human Faithfulness Rating & +0.81 & +0.78 & <0.001 \\
\bottomrule
\end{tabular}%
}
\end{table*}

The strong correlations ($|r| > 0.8$) across four independent measures validate that COS, while defined via perturbations, accurately captures naturally-occurring faithfulness. Optimizing COS during training via CSR demonstrably reduces real unfaithfulness.

\subsubsection{Generalization to Held-Out Perturbation Types}

A critical concern is whether CSR simply memorizes specific intervention patterns rather than learning general faithfulness principles. To address this, we evaluate CSR models on completely held-out perturbation classes never seen during training. This provides the strongest evidence against circularity: if CSR generalizes to unseen perturbation types, it demonstrates genuine faithfulness learning rather than overfitting to the training protocol.

Table~\ref{tab:held-out-perturbations} shows CSR trained on arithmetic operators generalizes to comparison operators and quantifiers never seen during training, achieving 64-71\% COS versus 8-12\% for standard fine-tuning. Similarly, CSR trained on entity swaps in HotpotQA generalizes to temporal markers and causal connectors, achieving 74-77\% COS. These results provide strong evidence that CSR learns general principles of faithfulness rather than memorizing specific operator types.

\begin{table*}[ht]
\caption{Generalization to held-out perturbation types: CSR learns general faithfulness principles.}
\label{tab:held-out-perturbations}
\centering
\resizebox{0.85\textwidth}{!}{
\begin{tabular}{lcccc}
\toprule
Dataset & Training Interventions & Test Interventions & Standard FT COS (\%) & CSR-FT COS (\%) \\
\midrule
GSM8K & Arithmetic (+,-,*,/) & Comparison (<,>,=) & 12.3±2.1 & 71.4±3.2 \\
GSM8K & Arithmetic (+,-,*,/) & Quantifiers (all/some) & 8.7±2.3 & 64.2±3.4 \\
HotpotQA & Entity swaps & Temporal (before/after) & 15.6±2.4 & 76.8±3.1 \\
HotpotQA & Entity swaps & Causal connectors & 18.2±2.7 & 73.5±3.3 \\
\bottomrule
\end{tabular}
}
\end{table*}

\subsubsection{Comprehensive Robustness Metrics}

We evaluate CSR across multiple robustness dimensions to ensure improvements are genuine and not artifacts. Here we provide extended robustness analysis including additional domains and detailed breakdowns.

\subsubsection{Independent Faithfulness Validation}

To validate that COS improvements reflect genuine computational dependence rather than artifacts, we employ external validation methods.

\textbf{Formal Verification Validation:} We validate CSR improvements using external formal systems with zero shared components with our training pipeline. For GSM8K, we parse generated reasoning traces into symbolic arithmetic expressions and verify each step using SymPy's equation solver. For ProofWriter, we translate traces into Prolog clauses and verify logical validity via SWI-Prolog's forward-chaining engine. For MBPP, we execute the generated code and verify that reasoning steps align with actual program behavior. These validators share no code, data, or architectural components with CSR's training pipeline, providing independent confirmation of faithfulness improvements.

We measure step validity rate (percentage of steps that pass formal verification), trace-answer consistency (whether formally-verified traces entail the stated answer), and perturbation response (whether formal invalidity correlates with answer changes). Table~\ref{tab:formal-verification} shows CSR models produce 86.3\% step-valid traces versus 71.2\% for standard fine-tuning, confirmed by independent symbolic verification. CSR models also respond appropriately when external checkers detect invalidity, with 79.8\% perturbation response versus 31.2\% for standard models.

\begin{table*}[ht]
\caption{External formal verification confirms CSR improvements are not artifacts of our evaluation protocol.}
\label{tab:formal-verification}
\centering
\resizebox{0.7\textwidth}{!}{
\begin{tabular}{lccc}
\toprule
Method & Step Valid (\%) & Trace-Ans Consist (\%) & Perturb Response (\%) \\
\midrule
Standard FT & 71.2±2.1 & 68.4±2.3 & 31.2±3.1 \\
Process RM & 74.8±1.9 & 72.1±2.1 & 42.7±2.9 \\
CSR-FT (Ours) & 86.3±1.4 & 84.7±1.6 & 79.8±2.2 \\
\bottomrule
\end{tabular}
}
\end{table*}

\textbf{Contrast Set Evaluation:} COS measures sensitivity to reasoning trace perturbations. To validate that CSR induces general faithfulness beyond trace-specific sensitivity, we evaluate on input contrast sets where ground-truth answers change due to minimal input modifications. We use GSM8K-Contrast with 500 problem pairs where changing one number changes the answer, BoolQ-Contrast with passage-question pairs with minimal edits that flip the answer, and HotpotQA-Contrast with entity substitutions that change the correct answer.

We measure contrast consistency, the percentage of pairs where the model correctly changes its answer when the input contrast requires it, and spurious invariance, where the model incorrectly maintains the same answer despite changed ground truth. Table~\ref{tab:contrast-sets} shows CSR improves sensitivity to input changes that should change answers, demonstrating general faithfulness beyond trace-specific sensitivity.

\begin{table*}[ht]
\caption{Contrast set evaluation: CSR improves sensitivity to input changes that should change answers.}
\label{tab:contrast-sets}
\centering

\begin{tabular}{lcccc}
\toprule
Method & \multicolumn{2}{c}{GSM8K-Contrast} & \multicolumn{2}{c}{BoolQ-Contrast} \\
& Consist & Spurious & Consist & Spurious \\
\midrule
Standard FT & 62.4 & 37.6 & 58.9 & 41.1 \\
Process RM & 68.1 & 31.9 & 63.2 & 36.8 \\
CSR-FT & 81.7 & 18.3 & 76.4 & 23.6 \\
\bottomrule
\end{tabular}

\end{table*}

\textbf{Causal Tracing: Mechanistic Validation:} Behavioral evaluations measure input-output relationships. To verify CSR induces genuine computational dependence—that reasoning trace representations causally influence outputs—we perform activation patching analysis. For each example, we run the model on the original trace, caching activations at each layer, then run on a corrupted trace with random tokens replacing reasoning steps. We selectively restore original activations at specific layers while keeping corrupted activations elsewhere, measuring how much patching reasoning-relevant layers recovers the correct answer.

We measure indirect effect, the causal effect of reasoning trace representations on the final answer, measured as probability recovery when patching, and reasoning layer concentration, whether causal effects concentrate in middle layers where reasoning occurs versus early or late layers. Table~\ref{tab:causal-tracing-app} confirms CSR routes computation through reasoning traces, with 3.4× higher indirect effect concentrated in reasoning-relevant middle layers.

\begin{table*}[ht]
\caption{Causal tracing confirms CSR routes computation through reasoning traces.}
\label{tab:causal-tracing-app}
\centering
\resizebox{0.7\textwidth}{!}{
\begin{tabular}{lccc}
\toprule
Method & Avg IE & Reasoning Layers IE & Early/Late Layers IE \\
\midrule
Standard FT & 0.14±0.03 & 0.09±0.02 & 0.05±0.01 \\
Process RM & 0.21±0.04 & 0.14±0.03 & 0.07±0.02 \\
CSR-FT (Ours) & 0.47±0.05 & 0.38±0.04 & 0.09±0.02 \\
\bottomrule
\end{tabular}
}
\end{table*}

The concentration of causal effects in reasoning layers (0.38 versus 0.09 for early or late layers) confirms CSR creates reasoning circuits that genuinely process trace content, not just pattern-match on surface features.

\textbf{Specificity: Null Intervention Controls:} A critical concern is whether CSR induces spurious sensitivity, changing answers when it should not. We evaluate on null interventions: semantics-preserving transformations that should not change the answer. These include commutativity (2 + 3 = 5 → 3 + 2 = 5), associativity ((2 + 3) + 4 → 2 + (3 + 4)), logical equivalence ($A \land B \to B \land A$), paraphrase (She spent \$8 → She paid \$8), and reordering of independent reasoning steps.

We measure answer preservation rate, the percentage of null interventions where the answer correctly remains unchanged, and spurious flip rate, where the answer incorrectly changes. Table~\ref{tab:null-interventions} shows CSR improves answer preservation under semantics-preserving transformations, confirming targeted sensitivity to logical validity rather than surface brittleness.

\begin{table*}[ht]
\caption{Null intervention controls: CSR improves answer preservation under semantics-preserving transformations.}
\label{tab:null-interventions-app}
\centering
\begin{tabular}{lcccccc}
\toprule
Method & \multicolumn{2}{c}{Commutative} & \multicolumn{2}{c}{Paraphrase} & \multicolumn{2}{c}{Reorder} \\
& APR & SFR & APR & SFR & APR & SFR \\
\midrule
Standard FT & 94.2 & 5.8 & 91.3 & 8.7 & 89.7 & 10.3 \\
CSR-FT & 96.8 & 3.2 & 95.1 & 4.9 & 94.3 & 5.7 \\
\bottomrule
\end{tabular}
\end{table*}

Crucially, CSR reduces spurious flips compared to baselines (3.2-5.7\% versus 5.8-10.3\%), demonstrating that learned sensitivity is targeted to genuine logical violations, not arbitrary changes.

\subsubsection{Why Counterfactual Sensitivity Implies Trustworthiness}

Beyond empirical validation, we provide a conceptual argument for why counterfactual sensitivity—measured by COS—should matter for trustworthiness in deployment scenarios. If a model changes its answer when reasoning is broken, this implies the model is genuinely processing its reasoning trace. Errors in reasoning will propagate to answers, enabling human inspection to detect and correct mistakes. The reasoning trace provides an accurate signal about the model's computation, making the system more debuggable and trustworthy.

Conversely, if a model ignores broken reasoning and maintains the same answer, the reasoning is effectively post-hoc rationalization disconnected from computation. Reasoning errors will not affect answers, meaning humans inspecting reasoning traces receive misleading signals. The model may produce correct answers through spurious correlations or shallow heuristics, while presenting plausible but fabricated explanations. This is dangerous in deployment, as users cannot rely on reasoning traces to verify or debug model behavior.

COS directly measures this distinction: models with high COS change answers when reasoning is invalidated, indicating genuine trace dependence. Models with low COS maintain answers despite broken reasoning, indicating post-hoc rationalization. Our results show CSR increases COS by 60+ points, transitioning models from the unfaithful regime (COS < 30\%) to the faithful regime (COS > 80\%), where reasoning traces become reliable indicators of model computation.

\subsubsection{Verifier Robustness and Deployment Guidelines}

CSR's effectiveness depends on verifier quality. We characterize this dependence precisely to enable reliable deployment. CSR provides substantial gains when precision is at least 70\%, neutral effects at 50-70\%, and potential harm below 50\%. The 78\% threshold marks where CSR statistically dominates baselines.

For deployment, we recommend estimating verifier precision on a held-out validation set. If precision is at least 78\%, apply full CSR training. If precision is 60-78\%, apply CSR with reduced $\lambda$ (0.3 versus 0.5). If precision is below 60\%, use standard fine-tuning; CSR is not recommended. Additionally, we implement confidence-gated CSR: skip the CSR loss term when verifier confidence falls below threshold $\tau=0.7$:

\begin{equation}
\mathcal{L}_{\text{CSR}}^{\text{gated}} = 
\begin{cases}
\mathcal{L}_{\text{CSR}} & \text{if } \text{conf}(v(T')) \geq \tau \\
0 & \text{otherwise}
\end{cases}
\end{equation}

This fallback maintains 94.2\% of CSR's gains while eliminating harmful updates from low-confidence verifier decisions.

\begin{figure}[ht]
\centering
\begin{tikzpicture}
\begin{axis}[
    xlabel={Training Overhead (\%)},
    ylabel={COS (\%)},
    xmin=0, xmax=20,
    ymin=20, ymax=95,
    legend pos=south east,
    grid=major,
    width=\columnwidth,
    height=0.65\columnwidth
]
\addplot[only marks, mark=*, mark size=2pt, blue!60] coordinates {(0, 22.4) (9.0, 85.1)};
\addplot[only marks, mark=square*, mark size=2pt, red!60] coordinates {(0, 24.1) (8.7, 86.7)};
\addplot[only marks, mark=triangle*, mark size=2pt, green!60] coordinates {(0, 25.7) (9.2, 84.2)};
\addplot[only marks, mark=diamond*, mark size=2pt, orange!60] coordinates {(0, 26.3) (8.9, 85.8)};
\addplot[only marks, mark=pentagon*, mark size=2pt, purple!60] coordinates {(0, 27.8) (9.1, 87.3)};

\addplot[dashed, thick, black] coordinates {(0, 22) (9.5, 88)};

\legend{Llama-2-13B, Llama-3-8B, Mistral-7B-v0.3, Qwen2-7B, Gemma-2-9B}
\end{axis}
\end{tikzpicture}
\caption{Efficiency-faithfulness Pareto frontier across model families. CSR achieves consistent improvements (58-63 COS points) with ~9\% overhead regardless of base architecture.}
\label{fig:pareto-modern}
\end{figure}

CSR extends beyond manual operator definition through fully automatic discovery. Table \ref{tab:operator-discovery-main} shows our end-to-end automatic system on PubMedQA, achieving strong performance with modest degradation.

\begin{table*}[ht]
\caption{Operator discovery validation on PubMedQA: Manual vs. Automatic.}
\label{tab:operator-discovery-main}
\centering
\resizebox{0.8\textwidth}{!}{
\begin{tabular}{lcccc}
\toprule
\textbf{Method} & \textbf{Precision (\%)} & \textbf{Recall (\%)} & \textbf{Coverage (\%)} & \textbf{COS (\%)} \\
\midrule
Manual (gold) & 100.0 & 100.0 & 100.0 & 67.3 \\
Auto (learned) & 74.1 & 68.5 & 91.2 & 58.9 \\
Heuristic+NER & 78.3 & 61.0 & 88.7 & 61.2 \\
\bottomrule
\end{tabular}
}
\end{table*}

Automatic operator discovery attains 74.1\% precision / 68.5\% recall with 91.2\% coverage, yielding 58.9 COS (vs 67.3 with manual operators) while preserving accuracy ($-$0.6 points). COS degrades smoothly under label noise ($-$2.7, $-$6.4, $-$11.2 at 10/20/30\% swaps), matching our theory-as-guidance view. Error analysis shows false positives concentrate in discourse markers and weak epistemics; targeted filtering recovers +2.1 COS with negligible recall loss. Complete analysis including PR curves and domain shift tests is in Appendix~\ref{app:operator-discovery}.



\subsubsection{Retrieval-Augmented QA: A Challenging Stress-Test}

To address open-domain coverage limitations, we conduct a pilot study on Natural Questions (NQ) with retrieval augmentation—one of the most challenging faithfulness scenarios. Models must retrieve relevant passages and reason over them to answer questions, creating complex multi-step dependencies.

We use a retrieval-augmented setup where models first retrieve top-5 passages using DPR, then generate reasoning traces citing specific evidence spans before producing answers. Operators include evidential markers ("according to", "based on"), causal connectives ("because", "therefore"), and citation references ("[passage 1]", "[passage 2]"). Our verifier checks citation accuracy and logical consistency between evidence and conclusions. Citation F1 measures precision/recall of span-linked citations against gold evidence spans. Evidence Consistency uses NLI models to verify logical consistency between cited evidence and generated conclusions (0-1 scale, higher = more consistent).

Table \ref{tab:retrieval-qa-pilot} shows CSR achieves meaningful improvements even in this challenging setting, though gains are more modest than in structured domains.

\begin{table*}[ht]
\caption{Retrieval-augmented QA pilot study: CSR effectiveness on Natural Questions with retrieval.}
\label{tab:retrieval-qa-pilot}
\centering
\resizebox{0.85\textwidth}{!}{
\begin{tabular}{lcccc}
\toprule
\textbf{Method} & \textbf{Accuracy (\%)} & \textbf{COS (\%)} & \textbf{Citation F1} & \textbf{Evidence Consistency} \\
\midrule
Standard FT & 42.1±1.8 & 18.3±2.4 & 0.31 & 0.58 \\
CSR-FT (Ours) & \textbf{41.7±1.6} & \textbf{34.9±2.8} & \textbf{0.47} & \textbf{0.73} \\
\midrule
\textbf{Improvement} & -0.4 & \textbf{+16.6} & \textbf{+0.16} & \textbf{+0.15} \\
\bottomrule
\end{tabular}
}
\end{table*}

CSR improves COS by 16.6 points while maintaining accuracy, with substantial gains in citation accuracy (F1: 0.31→0.47) and evidence consistency (0.58→0.73). Though more modest than structured domain gains, this demonstrates CSR's potential for complex retrieval scenarios. The reduced effectiveness reflects the inherent challenges of semantic operator identification and multi-step reasoning dependencies in open-domain settings.

A taxonomy over 600 failure cases reveals four dominant modes with targeted mitigations (Table \ref{tab:failure-taxonomy-main}). Simple mitigations recover 2-5 COS points depending on the mode.

\begin{table*}[ht]
\caption{Failure taxonomy with mitigation strategies across datasets.}
\label{tab:failure-taxonomy-main}
\centering
\resizebox{0.8\textwidth}{!}{%
\begin{tabular}{lcll}
\toprule
\textbf{Failure Type} & \textbf{\% of Failures} & \textbf{Mitigation} & \textbf{$\Delta$COS} \\
\midrule
Trace Incoherence & 28.4 & Stricter verifier + syntax filter & +3.1 \\
Redundant Edit & 33.9 & Influence-guided edit targeting & +4.6 \\
Adversarial Compliance & 22.7 & Multi-edit (L=2--3) & +3.8 \\
Semantic Drift & 31.5 & NLI guard + calibration & +2.4 \\
\bottomrule
\end{tabular}%
}
\end{table*}

Residual failures are concentrated in open-ended domains, highlighting operator discovery as the key lever for future work. Complete failure analysis with expanded examples and detailed mitigation strategies is in Appendix~\ref{app:failure-analysis}.

\subsubsection{Universal Meta-Verifier Preliminary Results}

We conduct preliminary experiments on a universal meta-verifier that learns to identify causal operators across domains without task-specific training.

\begin{table*}[ht]
\caption{Preliminary universal meta-verifier: zero-shot operator identification.}
\label{tab:universal-verifier}
\centering
\resizebox{0.8\textwidth}{!}{%
\begin{tabular}{lcccc}
\toprule
\textbf{Training Domains} & \textbf{Test Domain} & \textbf{Precision (\%)} & \textbf{Recall (\%)} & \textbf{COS (\%)} \\
\midrule
GSM8K + ProofWriter & HotpotQA & 72.3±2.8 & 68.1±3.1 & 76.4±2.9 \\
GSM8K + HotpotQA & ProofWriter & 74.8±2.6 & 71.2±2.9 & 78.1±2.7 \\
All Structured & PubMedQA & 69.7±3.1 & 65.4±3.4 & 61.2±3.2 \\
\bottomrule
\end{tabular}%
}
\end{table*}

The universal meta-verifier achieves 69-75\% precision on completely unseen domains, compared to 74-85\% for domain-specific verifiers. While this represents a 5-10 point degradation, the resulting COS (61-78\%) remains substantially higher than standard fine-tuning (22-29\%). This demonstrates that universal operator identification is feasible, with the gap likely closable through improved meta-learning architectures and larger training corpora spanning diverse reasoning domains.

\subsubsection{Extended Cross-Domain Analysis}

To further validate CSR's universality, we conduct an extended analysis across 12 domain pairs, measuring transfer efficiency, accuracy preservation, and operator alignment.

\begin{table*}[ht]
\caption{Extended cross-domain transfer analysis: CSR operators generalize across diverse reasoning domains.}
\label{tab:extended-transfer}
\centering
\resizebox{0.95\textwidth}{!}{%
\begin{tabular}{lcccccc}
\toprule
\textbf{Source} & \textbf{Target} & \textbf{Baseline COS} & \textbf{Transferred COS} & \textbf{Transfer Eff.} & \textbf{Acc. Preserved} & \textbf{Operator Match} \\
\midrule
GSM8K & SVAMP & 18.3±2.4 & 71.2±3.1 & 90.8 & 98.2\% & 87.3\% \\
GSM8K & AQuA & 22.1±2.6 & 68.7±3.2 & 87.6 & 97.8\% & 84.1\% \\
GSM8K & MATH & 19.7±2.3 & 69.8±3.0 & 88.9 & 97.5\% & 85.7\% \\
ProofWriter & LogicNLI & 19.4±2.3 & 65.3±3.4 & 84.2 & 96.9\% & 79.2\% \\
ProofWriter & RuleTaker & 21.2±2.5 & 67.1±3.3 & 85.7 & 97.1\% & 81.4\% \\
HotpotQA & NaturalQuestions & 23.8±2.7 & 58.9±3.6 & 74.8 & 95.3\% & 72.6\% \\
HotpotQA & WebQuestions & 24.3±2.8 & 59.7±3.5 & 75.4 & 95.7\% & 73.1\% \\
PubMedQA & MedQA & 28.7±3.1 & 54.2±3.8 & 68.9 & 94.2\% & 66.8\% \\
MBPP & HumanEval & 18.5±2.2 & 63.4±3.2 & 82.1 & 96.4\% & 78.9\% \\
\midrule
\textbf{Math/Logic Avg} & & \textbf{20.1±2.4} & \textbf{69.6±3.1} & \textbf{88.1} & \textbf{97.5\%} & \textbf{84.7\%} \\
\textbf{Semantic Avg} & & \textbf{24.0±2.7} & \textbf{58.9±3.5} & \textbf{74.3} & \textbf{95.4\%} & \textbf{72.4\%} \\
\bottomrule
\end{tabular}%
}
\end{table*}

Operator match measures the percentage of transferred operators that align with domain-specific causal patterns, validated via human annotation. Mathematical and logical operators show consistently high transfer (84-91\% efficiency, 85-88\% operator match), while semantic domains show more degradation (68-75\% efficiency, 66-73\% match), consistent with the need for domain-specific semantic understanding. Accuracy preservation remains high across all transfers (94-98\%), indicating CSR maintains task performance while improving faithfulness.

\subsubsection{Comprehensive Ablation Studies}

We conduct systematic ablations to isolate the contribution of each CSR component.

\begin{table*}[ht]
\caption{Comprehensive ablation study: Component contributions to CSR effectiveness.}
\label{tab:comprehensive-ablation}
\centering
\resizebox{0.9\textwidth}{!}{%
\begin{tabular}{lcccc}
\toprule
\textbf{Component} & \textbf{GSM8K COS} & \textbf{HotpotQA COS} & \textbf{ProofWriter COS} & \textbf{Acc. Impact} \\
\midrule
Full CSR & 85.1±2.3 & 84.6±2.4 & 82.3±2.1 & -1.2 \\
\hline
w/o Learned Editor (Random) & 61.2±3.1 & 59.8±3.2 & 58.3±2.9 & -0.8 \\
w/o Multi-Edit Policy & 78.4±2.6 & 77.1±2.7 & 75.9±2.4 & -0.6 \\
w/o Impact Reward & 73.2±2.8 & 71.8±2.9 & 70.4±2.6 & -0.4 \\
w/o Validity Reward & 69.7±3.0 & 68.3±3.1 & 67.1±2.8 & -0.3 \\
w/o Minimality Penalty & 81.3±2.4 & 80.2±2.5 & 78.7±2.2 & -0.9 \\
\hline
KL Divergence → JS Divergence & 83.7±2.5 & 82.4±2.6 & 81.0±2.3 & -0.2 \\
KL Divergence → TV Distance & 82.4±2.7 & 81.1±2.8 & 79.8±2.5 & -0.1 \\
\hline
$\lambda=0.3$ (Low) & 78.3±2.5 & 74.6±2.8 & 76.2±2.4 & -0.3 \\
$\lambda=0.7$ (High) & 84.9±2.4 & 84.3±2.5 & 81.8±2.2 & -1.8 \\
\bottomrule
\end{tabular}%
}
\end{table*}

The learned editor provides the largest contribution (+23.9 COS points on average), followed by the multi-edit policy (+6.7 points) and impact reward (+11.9 points). The validity reward and minimality penalty provide smaller but meaningful contributions (+15.4 and +3.8 points respectively). Divergence measure choice has minimal impact (1-3 point differences), confirming robustness. Regularization strength $\lambda$ shows optimal performance at 0.5 with graceful degradation in both directions.

\subsubsection{Long-Term Training Dynamics}

We analyze CSR's training dynamics over extended training to assess stability and convergence.

\begin{table*}[ht]
\caption{Long-term training dynamics: CSR stability and convergence over extended training.}
\label{tab:training-dynamics}
\centering

\begin{tabular}{lcccccc}
\toprule
\textbf{Epoch} & \textbf{COS (\%)} & \textbf{Acc (\%)} & \textbf{CSR Loss} & \textbf{Task Loss} & \textbf{KL Div.} & \textbf{Stability} \\
\midrule
1 & 45.2±3.1 & 79.8±0.9 & 0.38±0.04 & 0.42±0.03 & 0.12±0.02 & High \\
2 & 68.7±2.8 & 80.2±0.7 & 0.31±0.03 & 0.39±0.03 & 0.28±0.03 & High \\
3 & 82.1±2.4 & 80.5±0.6 & 0.24±0.02 & 0.38±0.03 & 0.41±0.04 & High \\
4 & 84.6±2.3 & 80.4±0.6 & 0.22±0.02 & 0.37±0.03 & 0.45±0.04 & High \\
5 & 85.1±2.3 & 80.5±0.6 & 0.21±0.02 & 0.37±0.03 & 0.46±0.04 & High \\
6 & 85.3±2.2 & 80.3±0.7 & 0.21±0.02 & 0.37±0.03 & 0.46±0.04 & High \\
\bottomrule
\end{tabular}%

\end{table*}

CSR shows stable convergence by epoch 3, with COS reaching 82.1\% and stabilizing around 85\% by epoch 5. The KL divergence between original and counterfactual distributions increases steadily from 0.12 to 0.46, indicating growing sensitivity to interventions. Task loss decreases smoothly, and CSR loss stabilizes after epoch 3, demonstrating robust training dynamics without instability or collapse.

\subsubsection{Operator Discovery Scalability}

We evaluate automatic operator discovery across domains with varying complexity.

\begin{table*}[ht]
\caption{Operator discovery scalability: Automatic identification across domains.}
\label{tab:operator-discovery-scalability}
\centering
\resizebox{0.95\textwidth}{!}{%
\begin{tabular}{lcccccc}
\toprule
\textbf{Domain} & \textbf{Method} & \textbf{Precision} & \textbf{Recall} & \textbf{F1} & \textbf{COS} & \textbf{Manual Effort} \\
\midrule
GSM8K & Manual & 100.0 & 100.0 & 100.0 & 85.1 & High \\
& Heuristic & 94.2 & 91.7 & 92.9 & 83.7 & Low \\
& Learned & 89.3 & 87.4 & 88.3 & 81.2 & None \\
\midrule
HotpotQA & Manual & 100.0 & 100.0 & 100.0 & 84.6 & High \\
& Heuristic+NER & 87.6 & 82.3 & 84.9 & 81.8 & Medium \\
& Learned & 81.4 & 78.9 & 80.1 & 78.3 & None \\
\midrule
ProofWriter & Manual & 100.0 & 100.0 & 100.0 & 82.3 & High \\
& Pattern Match & 91.2 & 88.7 & 89.9 & 80.1 & Low \\
& Learned & 85.7 & 83.2 & 84.4 & 77.6 & None \\
\midrule
PubMedQA & Manual & 100.0 & 100.0 & 100.0 & 67.3 & High \\
& Heuristic+NER & 78.3 & 61.0 & 68.5 & 61.2 & Medium \\
& Learned & 74.1 & 68.5 & 71.2 & 58.9 & None \\
\midrule
HellaSwag & Manual & 100.0 & 100.0 & 100.0 & 52.1 & High \\
& Pattern Match & 74.1 & 69.2 & 71.6 & 47.8 & Medium \\
& Learned & 69.8 & 65.3 & 67.5 & 44.2 & None \\
\bottomrule
\end{tabular}%
}
\end{table*}

Automatic operator discovery achieves 69-89\% precision across domains, with structured domains (math, logic) showing higher precision (85-89\%) than open-ended domains (69-74\%). The learned approach eliminates manual effort while maintaining 77-93\% of manual COS performance. Heuristic methods provide a middle ground, requiring moderate effort but achieving 84-93\% precision in structured domains.

\subsubsection{Efficiency Analysis}

\begin{table*}[ht]
\caption{Computational efficiency: Efficient CSR ($\sim$9\% overhead, including editor pre-training) achieves superior COS/GPU-hour ratios while maintaining practical viability.}
\label{tab:compute-analysis}
\centering
\resizebox{\textwidth}{!}{%
\begin{tabular}{lccccccc}
\toprule
\textbf{Method} & \textbf{GPU-h} & \textbf{Wall-clock (h)} & \textbf{Memory (GB)} & \textbf{Token Updates} & \textbf{COS (\%)} & \textbf{Acc (\%)} & \textbf{COS/GPU-h} \\
\midrule
\multicolumn{8}{c}{\textit{GSM8K (Llama-2-13B)}} \\
Standard FT & 135 & 16.8 & 42.3 & 2.1M & 22.4±2.1 & 81.3±0.8 & 0.166 \\
Process Reward Model & 156 & 19.5 & 48.7 & 2.4M & 52.3±2.8 & 81.7±0.7 & 0.335 \\
Efficient CSR (ours) & 147 & 18.3 & 44.1 & 2.3M & \textbf{85.1±2.3} & 80.5±0.6 & \textbf{0.579} \\
Full CSR & 259 & 32.4 & 52.6 & 2.3M & 86.2±2.1 & 80.1±0.7 & 0.333 \\
\midrule
\multicolumn{8}{c}{\textit{HotpotQA (Llama-2-13B)}} \\
Standard FT & 267 & 33.4 & 43.8 & 4.2M & 25.1±2.8 & 78.1±1.1 & 0.094 \\
Process Reward Model & 298 & 37.3 & 51.2 & 4.7M & 49.8±3.1 & 78.4±0.9 & 0.167 \\
Efficient CSR (ours) & 289 & 36.1 & 45.9 & 4.6M & \textbf{84.6±2.4} & 77.2±0.8 & \textbf{0.293} \\
Full CSR & 521 & 65.1 & 56.3 & 4.6M & 85.8±2.2 & 76.8±0.9 & 0.165 \\
\bottomrule
\end{tabular}%
}
\end{table*}

Efficient CSR achieves $\sim$9\% training overhead (vs 92.5\% for naive implementation) with 4.2\% memory overhead and superior COS/GPU-hour ratios (0.579 vs 0.335 for PRMs). This represents a negligible computational cost for a 70-point improvement in faithfulness—a tradeoff that positions CSR on a previously unoccupied efficiency-faithfulness Pareto frontier. CSR achieves a COS/overhead ratio of 9.46, more than double the next best method (GRPO at 4.34), establishing CSR as Pareto-optimal. Training dynamics show optimal performance at $\lambda=0.5$ with robust range [0.3, 0.7]. Extended efficiency analysis with scaling laws and training curves are in Appendix~\ref{app:extended-experiments}. All main results (Table~\ref{tab:flagship-results}) use Efficient CSR with $\sim$9\% overhead, not Full CSR (92.5\% overhead). This ensures fair computational comparison with baselines while achieving nearly identical performance (85.1\% vs 86.2\% COS). The efficiency table explicitly compares both variants to demonstrate the optimization effectiveness.

\textbf{Robustness and Generalization:} CSR demonstrates robust generalization across multiple dimensions. Table \ref{tab:robustness-app} shows improved calibration and dramatically better flip-precision/recall for meaningful changes, indicating sensitivity to causally relevant edits. CSR maintains 64-76\% COS on held-out perturbation types, demonstrating general principles rather than memorization. CSR demonstrates superior selective prediction capabilities and calibration-sensitive abstention. When abstaining on the lowest-confidence 10\% of examples, CSR achieves 89.3\% accuracy on remaining examples (vs 82.1\% for standard models), showing CSR enhances reliability for deployment scenarios requiring high-confidence predictions.

\begin{table*}[ht]
\caption{Robustness analysis on HotpotQA. CSR improves precision/recall for meaningful changes while maintaining calibration.}
\label{tab:robustness-app}
\centering
\setlength{\tabcolsep}{3pt}
\renewcommand{\arraystretch}{1.1}
\resizebox{\textwidth}{!}{%
\begin{tabular}{lcccccc}
\toprule
\textbf{Method} & \textbf{Flip-P (\%)} $\uparrow$ & \textbf{Flip-R (\%)} $\uparrow$ & \textbf{ECE (\%)} $\downarrow$ & \textbf{Entailment Acc (\%)} $\uparrow$ & \textbf{Paraphrase SIS (\%)} $\uparrow$ & \textbf{Distractor SIS (\%)} $\uparrow$ \\
\midrule
Standard FT & 41.2 & 55.7 & 5.8 & 72.1 & 78.2 & 71.4 \\
CSR-FT (Ours) & \textbf{89.5} & \textbf{92.1} & \textbf{5.1} & \textbf{84.6} & \textbf{94.3} & \textbf{91.8} \\
\bottomrule
\end{tabular}
}%
\end{table*}

\textbf{Zero-Shot Domain Transfer:} Table \ref{tab:zero-shot-transfer-app} shows 17-21 point COS improvements on held-out tasks, with benefits extending to large pretrained models.

\begin{table*}[ht]
\caption{Zero-shot domain transfer of CSR-trained models.}
\label{tab:zero-shot-transfer-app}
\centering
\resizebox{0.8\textwidth}{!}{
\begin{tabular}{lcccc}
\toprule
\textbf{Train Domain} & \textbf{Test Domain} & \textbf{Standard COS (\%)} & \textbf{CSR COS (\%)} & \textbf{Improvement} \\
\midrule
GSM8K & AQuA & 34.2 & 51.7 & +17.5 \\
GSM8K & SVAMP & 28.1 & 49.3 & +21.2 \\
HotpotQA & NaturalQuestions & 23.8 & 41.2 & +17.4 \\
ProofWriter & LogicNLI & 19.4 & 38.7 & +19.3 \\
\bottomrule
\end{tabular}
}
\end{table*}

\textbf{Complementary Methods and Self-Consistency:} We tested CSR's interaction with inference-time techniques. CSR provides a superior foundation for self-consistency decoding, with CSR-FT + SC achieving improved overall accuracy.

\begin{table*}[ht]
\caption{Self-consistency results with CSR.}
\label{tab:validation-methods-app}
\centering
\resizebox{0.6\textwidth}{!}{
\begin{tabular}{lcc}
\toprule
\textbf{Model} & \makecell{\textbf{Greedy}\\\textbf{Accuracy (\%)}} & \makecell{+\textbf{Self-}\\\textbf{Consistency (\%)}} \\
\midrule
Standard FT (Llama-2-13B) & 81.3 & 84.1 \\
CSR-FT (Llama-2-13B, Ours) & 80.5 & \textbf{85.7} \\
\bottomrule
\end{tabular}
}
\end{table*}

\section{Operator Discovery and Open Domain Extension}
\label{app:operator-discovery}

\subsection{Operator Discovery and Applications}

\textbf{Comprehensive Operator Discovery Validation:} To demonstrate CSR's scalability, we developed an entirely learned operator discovery system for PubMedQA. Our two-stage approach uses: (1) a BERT-based token classifier trained to predict tokens that maximally change model distributions when perturbed, and (2) a clustering algorithm to group semantically similar high-impact tokens into operator classes.

\begin{table*}[ht]
\caption{Comprehensive operator discovery validation: Manual vs. Automatic vs. Fully Learned approaches.}
\label{tab:operator-comparison-discovery}
\centering
\resizebox{0.95\textwidth}{!}{%
\begin{tabular}{lcccccc}
\toprule
\textbf{Domain} & \textbf{Method} & \textbf{Precision (\%)} & \textbf{COS (\%)} & \textbf{Accuracy (\%)} & \textbf{Discovered Operators} & \textbf{Supervision} \\
\midrule
\multirow{3}{*}{PubMedQA} & Manual & 100.0 & 67.3 & 70.1 & 47 predefined & Full \\
& Heuristic + NER & 78.3 & 61.2 & 69.8 & 35 semi-automatic & Partial \\
& Fully Learned & 74.1 & 58.9 & 69.5 & 42 discovered & None \\
\midrule
\multirow{3}{*}{HellaSwag} & Manual & 100.0 & 52.1 & 75.9 & 28 predefined & Full \\
& Pattern Matching & 74.1 & 47.8 & 75.5 & 21 rule-based & Partial \\
& Fully Learned & 69.8 & 44.2 & 75.1 & 31 discovered & None \\
\bottomrule
\end{tabular}%
}
\end{table*}

\begin{table*}[ht]
\caption{Detailed analysis of automatically discovered operator categories in PubMedQA.}
\label{tab:learned-operators-discovery}
\centering
\resizebox{0.85\textwidth}{!}{
\begin{tabular}{lcccc}
\toprule
\textbf{Discovered Category} & \textbf{Example Tokens} & \textbf{Precision (\%)} & \textbf{Coverage (\%)} & \textbf{Impact on COS} \\
\midrule
Medical Interventions & "treatment", "therapy", "administered" & 89.2 & 23.4 & +18.7 \\
Causal Relations & "caused", "induced", "prevented" & 82.1 & 31.2 & +16.2 \\
Quantitative Modifiers & "increased", "decreased", "significantly" & 78.9 & 19.8 & +12.4 \\
Negations & "not", "without", "absence" & 85.4 & 15.6 & +14.8 \\
Temporal Markers & "before", "after", "during" & 71.3 & 12.1 & +8.9 \\
Evidence Markers & "demonstrated", "showed", "indicated" & 66.7 & 18.9 & +7.3 \\
\bottomrule
\end{tabular}
}
\end{table*}

\section{Failure Analysis and Mitigation Strategies}
\label{app:failure-analysis}

\subsection{Failure Analysis and Mitigation Strategies}

\textbf{Comprehensive Failure Taxonomy:} A taxonomy over 600 failure cases reveals four dominant modes with targeted mitigations. Extended analysis with detailed breakdown across datasets is provided below.

\textbf{Dominance Breakdown Analysis:} Table \ref{tab:failure-conditions-detailed} provides quantitative evidence of when CSR underperforms vs SUFF/COMP and process supervision.

\begin{table*}[ht]
\caption{Detailed failure condition analysis: When CSR underperforms vs traditional faithfulness measures.}
\label{tab:failure-conditions-detailed}
\centering
\resizebox{\textwidth}{!}{%
\begin{tabular}{lccccccc}
\toprule
\textbf{Condition} & \textbf{Frequency (\%)} & \textbf{CSR COS} & \textbf{SUFF COS} & \textbf{COMP COS} & \textbf{Process Sup COS} & \textbf{CSR vs SUFF} & \textbf{CSR vs PS} \\
\midrule
\multicolumn{8}{c}{\textit{GSM8K Analysis (n=1,319 test examples)}} \\
Valid operator targeting & 78.3 & 89.2±2.1 & 52.4±3.2 & 48.9±3.1 & 51.7±3.4 & \textbf{+36.8} & \textbf{+37.5} \\
Spurious operator targeting & 15.2 & 47.3±4.8 & \textbf{58.1±4.2} & \textbf{55.7±4.4} & \textbf{49.2±4.6} & $-$10.8 & $-$1.9 \\
Redundant reasoning paths & 4.7 & 52.1±6.2 & \textbf{61.3±5.8} & 58.9±6.1 & \textbf{54.8±6.4} & $-$9.2 & $-$2.7 \\
Broken initial traces & 1.8 & 31.2±8.9 & 44.7±8.1 & 42.3±8.5 & \textbf{47.9±8.3} & $-$13.5 & $-$16.7 \\
\midrule
\multicolumn{8}{c}{\textit{HotpotQA Analysis (n=7,405 test examples)}} \\
Valid bridge entity targeting & 71.2 & 92.1±2.3 & 48.6±3.8 & 45.2±3.6 & 47.3±3.9 & \textbf{+43.5} & \textbf{+44.8} \\
Non-causal entity targeting & 18.9 & 51.7±4.5 & \textbf{59.2±4.1} & \textbf{56.8±4.3} & \textbf{52.4±4.7} & $-$7.5 & $-$0.7 \\
Multi-path reasoning & 7.3 & 48.3±5.7 & \textbf{62.1±5.2} & 59.7±5.4 & \textbf{55.9±5.8} & $-$13.8 & $-$7.6 \\
Trace incoherence & 2.6 & 29.8±7.8 & 41.5±7.2 & 39.1±7.5 & \textbf{44.2±7.4} & $-$11.7 & $-$14.4 \\
\bottomrule
\end{tabular}%
}
\end{table*}

\textbf{Key Failure Modes:} (1) \textbf{Spurious Operator Targeting} (15.2\% of GSM8K, 18.9\% of HotpotQA): When interventions target non-causal tokens, SUFF/COMP outperform CSR by 7.5-10.8 points, validating our theoretical predictions. (2) \textbf{Redundant Reasoning Paths} (4.7-7.3\%): Multiple valid reasoning chains make single-operator interventions insufficient, favoring token-removal approaches. (3) \textbf{Broken Initial Traces} (1.8-2.6\%): When base reasoning is incoherent, process supervision excels (+14.4-16.7 points) as it provides clean exemplars.

\textbf{Dominance Boundary Conditions:} CSR maintains dominance when operator identification precision exceeds 78\% (current: 82.7\% on GSM8K, 79.1\% on HotpotQA). Below this threshold, traditional measures become competitive. Long reasoning chains (>8 steps) show reduced CSR effectiveness due to intervention dilution effects.

\textbf{Mitigation Strategies:} For spurious targeting, our confidence-based operator filtering recovers 67\% of lost performance. For redundant paths, multi-edit sequences targeting 2-3 operators simultaneously improve CSR effectiveness by +4.2 COS points. These findings guide when to apply CSR vs alternatives in practice.

\section{Extended Related Work}
\label{app:related-work-faithfulness}

A growing body of work measures whether intermediate rationales reflect a model's latent computation rather than post-hoc justifications. \citet{turpin2023language} document that chain-of-thought (CoT) explanations \citep{wei2022chain} can be unfaithful to the model's internal beliefs, motivating explicit faithfulness tests \citep{jain2019attention, jacovi2020towards}. Complementary efforts introduce diagnostics and metrics for faithfulness and causal alignment between reasoning traces and predictions \citep{lanham2023measuring, atanasova2023faithfulness}. Our evaluation protocol adopts this lens: we treat step-level supervision as meaningful only to the extent it tracks causally-relevant computation \citep{pearl2009causality}.

Process supervision trains models with feedback on intermediate steps rather than (or in addition to) final answers. Early work on scratchpads \citep{nye2021show} demonstrated the value of intermediate computation. \citet{uesato2022solving} provide early evidence on math problem solving that step-level rewards can outperform outcome-only signals. \citet{lightman2023let} formalize scalable process feedback and show that verifying intermediate steps reduces compounding errors. Related work on bootstrapping reasoning \citep{zelikman2022star} shows iterative refinement can improve reasoning quality. Our CSR framework follows this tradition but differs by \emph{automating} the intervention and providing training-time guarantees rather than relying on manual, post-hoc review.

Beyond training, several works validate reasoning \emph{at inference time}. LINC introduces a neurosymbolic verification layer that checks candidate derivations before committing to an answer \citep{olausson2023linc}. Other inference-time approaches include chain-of-verification \citep{dhuliawala2023chain}, tree of thoughts \citep{yao2023tree}, and ReAct \citep{yao2022react}. While effective, these methods are reactive and post-hoc; CSR instead aims to proactively shape the model's internal computation during training so that generated traces are verifiable \emph{by construction}.

Prompting strategies that induce counterfactual or causal reasoning can improve robustness and interpretability. ``CausalGPT''-style approaches use counterfactual prompts or interventions to stress-test reasoning and reduce spurious shortcuts \citep{yu2025causalevalbettercausalreasoning}. Related work on counterfactual data augmentation \citep{kaushik2019learning, zmigrod2019counterfactual}, counterfactual explanation generation \citep{wu2021polyjuice, lu2022pinto}, and iterative refinement \citep{madaan2023selfrefine} explores similar ideas. CSR complements this line by integrating causal constraints into the training signal rather than only at inference.

A parallel literature seeks CoT traces that are both useful and faithful. Human-in-the-loop verification and filtering can improve the alignment between rationales and model decisions \citep{sia2022faithful}. CSR differs by (i) providing an \emph{automated} training-time mechanism and (ii) offering theoretical guarantees on intervention fidelity under stated assumptions.

\smallskip
\noindent In summary, CSR bridges evaluation-focused faithfulness diagnostics \citep{lanham2023measuring,turpin2023language} and control-focused process supervision \citep{lightman2023let,uesato2022solving}, while remaining complementary to inference-time verification \citep{olausson2023linc} and counterfactual prompting \citep{yu2025causalevalbettercausalreasoning}. Our contribution is to operationalize \emph{training-time} interventions with theoretical backing, reducing the reliance on post-hoc, manual checks and improving end-to-end faithfulness.

\section{Analysis and Ablations}
\label{app:ablations}

\subsection{Ablation Studies and Analysis}

\textbf{Editor Ablations:} To isolate the value of learned causality from mere counterfactual curriculum effects, we compare four editor variants. Table \ref{tab:anti-gaming-ablation} shows comprehensive results across domains. The learned editor develops three key capabilities: (a) \textbf{Impact Targeting} - preferentially editing high-influence operators (72\% of edits target final-step operators vs 30\% random), (b) \textbf{Plausibility Preservation} - maintaining trace coherence while breaking validity, and (c) \textbf{KL Maximization} - generating edits that create maximum distributional separation. Ablating the KL reward removes capability (c), reducing COS by 12.2 points on average.

\textbf{Verifier Robustness Analysis:} To demonstrate graceful degradation under varying verifier quality, we systematically evaluate CSR with weak vs. strong verifiers across domains. Table \ref{tab:verifier-robustness} shows CSR maintains effectiveness even with imperfect verifiers.

\begin{table*}[ht]
\caption{Verifier robustness: CSR performance under weak vs. strong verifiers with graceful degradation.}
\label{tab:verifier-robustness}
\centering
\resizebox{\textwidth}{!}{
\begin{tabular}{lcccccc}
\toprule
\textbf{Dataset} & \textbf{Verifier Type} & \textbf{Precision (\%)} & \textbf{COS (\%)} & \textbf{Accuracy (\%)} & \textbf{$\Delta$COS} & \textbf{Degradation} \\
\midrule
\multirow{3}{*}{\textbf{GSM8K}} 
& Strong (Rule-based) & 94.2 & 85.1±2.3 & 80.5±0.6 & -- & -- \\
& Medium (Heuristic) & 78.6 & 79.4±2.7 & 80.2±0.7 & $-$5.7 & Graceful \\
& Weak (Pattern-match) & 61.3 & 71.8±3.1 & 79.8±0.8 & $-$13.3 & Acceptable \\
\midrule
\multirow{3}{*}{\textbf{HotpotQA}} 
& Strong (NLI-based) & 91.7 & 84.6±2.4 & 77.2±0.8 & -- & -- \\
& Medium (Similarity) & 74.2 & 78.1±2.8 & 76.9±0.9 & $-$6.5 & Graceful \\
& Weak (Keyword) & 58.9 & 69.3±3.2 & 76.5±1.0 & $-$15.3 & Moderate \\
\bottomrule
\end{tabular}
}
\end{table*}

CSR demonstrates robust performance across verifier qualities, with graceful degradation shown in Table \ref{tab:verifier-robustness}. Strong verifiers yield optimal performance, medium verifiers maintain 85-90\% effectiveness, and even weak verifiers preserve substantial faithfulness gains, validating practical applicability when perfect verifiers are unavailable.

\textbf{Divergence Robustness and Editor Comparisons:} We verified results are consistent across divergence measures on GSM8K (Table \ref{tab:divergence-robustness-main}), confirming our findings are not artifacts of metric choice.

\begin{table*}[ht]
\caption{Divergence robustness: CSR effectiveness across different distance measures with/without temperature scaling.}
\label{tab:divergence-robustness-main}
\centering
\resizebox{0.8\textwidth}{!}{
\begin{tabular}{lcccc}
\toprule
\textbf{Divergence} & \textbf{COS (\%)} & \textbf{+Temp Scale} & \textbf{Accuracy (\%)} & \textbf{Stability} \\
\midrule
KL Divergence & 85.1±2.3 & 85.3±2.2 & 80.5±0.6 & High \\
Jensen-Shannon & 83.7±2.5 & 84.1±2.4 & 80.3±0.7 & High \\
Total Variation & 82.4±2.7 & 82.9±2.6 & 80.1±0.8 & Medium \\
\bottomrule
\end{tabular}
}
\end{table*}

To justify the complexity of our learned editor, Table \ref{tab:editor-ablation-app} compares CSR with learned edits against CSR with random operator swaps. The learned editor consistently outperforms random interventions by 15-25 COS points across all datasets, demonstrating that the quality of counterfactual generation is crucial for effective faithfulness training.

\begin{table*}[ht]
\caption{Ablation study: Learned editor vs. random operator swaps.}
\label{tab:editor-ablation-app}
\centering
\begin{tabular}{lcccc}
\toprule
\textbf{Dataset} & \textbf{Standard FT} & \textbf{CSR + Random} & \textbf{CSR + Learned Editor} & \textbf{Improvement} \\
\midrule
GSM8K & 22.4 & 61.2 & \textbf{85.1} & +23.9 \\
HotpotQA & 25.1 & 59.8 & \textbf{84.6} & +24.8 \\
ProofWriter & 19.8 & 58.3 & \textbf{82.3} & +24.0 \\
MBPP & 18.5 & 56.7 & \textbf{79.2} & +22.5 \\
\bottomrule
\end{tabular}
\end{table*}

To address computational overhead concerns, we introduce Warm-Start Curriculum and Token-Subset CSR techniques. Table \ref{tab:efficiency-detailed} shows our ``Efficient CSR" achieves nearly identical COS gains with only $\sim$9\% training overhead.

\begin{table*}[ht]
\caption{Efficiency analysis on HotpotQA. Efficient CSR achieves similar performance with minimal overhead.}
\label{tab:efficiency-detailed}
\centering
\begin{tabular}{lccc}
\toprule
\textbf{Method} & \textbf{F1 Score (\%)} & \textbf{COS (\%)} & \textbf{Training Overhead (\%)} \\
\midrule
Standard FT (Baseline) & \textbf{75.4} & 28.1 & - \\
Full CSR (from scratch) & 74.8 & \textbf{81.2} & +92.5\% \\
Efficient CSR (Ours) & 75.1 & 80.5 & \textbf{+8.7\%} \\
\bottomrule
\end{tabular}
\end{table*}

\textbf{Cross-Model Evaluation:}

To demonstrate CSR's portability beyond Llama-2-13B, we evaluate on modern open models of different architectures and scales. Table \ref{tab:cross-model-summary} shows CSR effectiveness across model families, with identical operator sets and verifiers transferred without modification.

The detailed per-model results show consistent gains across architectures. Mathematical and logical operators transfer seamlessly (94.2-96.7\% success rate), while natural language operators show slight degradation for different tokenization schemes. Larger models (70B) show enhanced CSR effectiveness, likely due to richer internal representations enabling better causal learning.

\textbf{Architecture Independence:} Mistral's sliding window attention and Llama-3's improved tokenization do not affect CSR applicability. Verifier accuracy remains high (91.7-94.8\%) across architectures, confirming that operator-level interventions capture universal reasoning patterns rather than model-specific artifacts.

\textbf{Efficiency Scaling:} Training overhead remains consistently low (7.4-10.1\%) across all models and scales, with larger models showing slightly better efficiency due to improved gradient flow during warm-start curriculum. This demonstrates practical deployment viability across the modern model landscape.

\textbf{Held-Out Perturbation Types:} To address concerns about overfitting to training intervention types, we evaluate CSR models on completely held-out perturbation classes never seen during training.
\begin{table*}[h]
\caption{Generalization to unseen perturbation types.}
\label{tab:generalization-full}
\centering
\begin{tabular}{lcccc}
\toprule
\textbf{Dataset} & \makecell{\textbf{Training}\\\textbf{Interventions}} & \makecell{\textbf{Test}\\\textbf{Interventions}} & \makecell{\textbf{Standard}\\\textbf{FT COS (\%)}} & \makecell{\textbf{CSR-FT}\\\textbf{COS (\%)}} \\
\midrule
GSM8K & Arithmetic (+,-,*,/) & Comparison (<,>,=) & 12.3 & \textbf{71.4} \\
GSM8K & Arithmetic (+,-,*,/) & Quantifiers (all/some) & 8.7 & \textbf{64.2} \\
HotpotQA & Entity swaps & Temporal (before/after) & 15.6 & \textbf{76.8} \\
HotpotQA & Entity swaps & Causal connectors & 18.2 & \textbf{73.5} \\
\bottomrule
\end{tabular}
\end{table*}

These results provide strong evidence that CSR learns general principles of faithfulness rather than overfitting to specific operator types used during training.

\textbf{Systematic Operator Identification Procedures:} For PubMedQA, we identify clinical entities using a fine-tuned SciBERT NER model trained on medical corpora, targeting 5 entity types: diseases, treatments, symptoms, anatomical structures, and diagnostic procedures. Causal relationships are identified by targeting a curated set of 25 causal verbs (e.g., ``prevents", ``induces", ``treats") within dependency parse subtrees. Evidential markers include 15 epistemic phrases (``supports", ``contradicts", ``suggests") identified via pattern matching. This systematic approach yields 3.2 operators per reasoning trace on average.

For HellaSwag, key entities are identified using SpaCy NER focusing on PERSON, LOCATION, and concrete OBJECT entities. Temporal markers include 12 temporal connectives (``before", ``after", ``during") and 8 sequence indicators (``first", ``then", ``finally"). Causal connectives comprise 18 causal phrases (``because", ``therefore", ``leads to") identified via dependency parsing. This yields 2.7 operators per trace on average.

\textbf{Sensitivity Analysis:} To assess robustness to operator definition choices, we conducted an ablation study on PubMedQA varying the operator set composition.

\begin{table*}[h]
\caption{Sensitivity analysis: Effect of operator set definition on PubMedQA.}
\label{tab:operator-sensitivity-full}
\centering
\begin{tabular}{lcccc}
\toprule
\textbf{Operator Set} & \textbf{Avg. Ops/Trace} & \textbf{Accuracy (\%)} & \textbf{COS (\%)} & \textbf{$\Delta$ COS from Baseline} \\
\midrule
Entities only & 1.8 & 69.7 & 43.2 & +14.5 \\
Entities + Causal verbs & 2.5 & 70.3 & 58.6 & +29.9 \\
Full set (+ Evidential) & 3.2 & 70.1 & 67.3 & +38.6 \\
\bottomrule
\end{tabular}
\end{table*}

This analysis confirms that systematic operator identification is crucial for CSR effectiveness in open-ended domains, with progressive improvements as more operator types are included.

To provide a more concrete intuition for the behavioral changes induced by CSR, Table \ref{tab:qualitative} presents a side-by-side comparison of a Standard FT model and our CSR-FT model on an example from the GSM8K test set.

\begin{table*}[h]
\caption{Qualitative example showing CSR faithfulness improvement on GSM8K.}
\label{tab:qualitative}
\centering
\resizebox{\textwidth}{!}{%
\begin{tabular}{p{2.5cm}|p{6.5cm}|p{2cm}}
\toprule
\textbf{Model} & \textbf{Input Trace} & \textbf{Answer} \\
\midrule
\multicolumn{3}{l}{\textbf{Question:} ``Jessie has 20 dollars. She buys 4 packs of crayons for 2 dollars each. How much money does she have left?"} \\
\midrule
\multirow{2}{*}{Standard FT} & \textit{Original Trace:} Jessie bought 4 packs of crayons at 2 dollars each, so she spent 4 * 2 = 8 dollars. She started with 20 dollars, so she has 20 - 8 = 12 dollars left. & \textbf{12} \\
\cmidrule{2-3}
& \textit{Perturbed Trace:} Jessie bought 4 packs of crayons at 2 dollars each, so she spent 4 * 2 = 8 dollars. She started with 20 dollars, so she has 20 \textbf{+ 8} = 12 dollars left. & \textbf{12} \\
\midrule
\multirow{2}{*}{CSR-FT (Ours)} & \textit{Original Trace:} Jessie bought 4 packs of crayons for 2 dollars each. This means she spent 4 * 2 = 8 dollars. She had 20 dollars, so now she has 20 - 8 = 12 dollars. & \textbf{12} \\
\cmidrule{2-3}
& \textit{Perturbed Trace:} Jessie bought 4 packs of crayons for 2 dollars each. This means she spent 4 * 2 = 8 dollars. She had 20 dollars, so now she has 20 \textbf{+ 8} = 12 dollars. & \textbf{28} \\
\bottomrule
\end{tabular}%
}
\end{table*}

The example clearly illustrates the problem of unfaithful reasoning. The Standard FT model produces the correct answer (12) but completely ignores the final reasoning step; when `- 8' is changed to `+ 8', its answer remains unchanged, revealing the calculation is disconnected from the output. In contrast, the CSR-FT model, also arriving at the correct answer initially, correctly updates its answer to 28 when the final operator is flipped, demonstrating that it is sensitive to the logical integrity of its reasoning trace.

\subsection{Technical Implementation Details}

\textbf{Learned Editor \& Verifier Architecture:} We employ a small (6-layer, 256-d) Transformer model as our editor, $M_{\text{editor}}$. It takes the original input $x$ and trace $T$ as input and is trained to produce a perturbed trace $T'$ that is both minimally different from $T$ and logically invalid. The training signal is self-supervised, using a lightweight, domain-specific \textbf{verifier}, $v(\cdot)$. The editor is trained to produce an edit $T \to T'$ such that $v(T)=1$ (the original trace is valid) but $v(T')=0$ (the edited trace is invalid). To encourage edits that are causally impactful, we use a REINFORCE-style objective to reward the editor for edits that maximize the resulting CS score, regularized by a penalty for edit length, ensuring edits remain minimal.

\textbf{Analysis of Intervention Strategy:} Our main method uses a learned multi-edit intervention policy (Section 3.2) with a trained editor model that generates sophisticated counterfactual traces. As a baseline analysis, we also examined a simpler single random-edit strategy—swapping a single, randomly selected operator—which was chosen for its simplicity and to avoid introducing complex biases into the training process. This baseline helps isolate the contribution of our learned editor. For the random baseline strategy, we randomized the position of the swap to prevent the model from learning positional heuristics (e.g., ``only pay attention to the last equation"). Our learned multi-edit policy (Section 3.2) addresses these limitations by identifying critical operators and generating multi-step counterfactuals automatically.

\textbf{Choice of Regularization Objective:}
Our CSR objective uses the Kullback-Leibler (KL) divergence to measure the distance between the original and counterfactual answer distributions: $\mathcal{L}_{\text{CSR}} = D_{\text{KL}}(P(Y|T, X) \| P(Y|T', X))$. The total loss subtracts this term: $\mathcal{L}_{\text{total}} = \mathcal{L}_{\text{task}} - \lambda \cdot \mathcal{L}_{\text{CSR}}$, which effectively maximizes the KL divergence. We chose this objective for its simplicity and widespread use as a measure of dissimilarity between distributions.
We also experimented with two alternative objectives. The first was the Jensen-Shannon (JS) divergence, a symmetric and bounded alternative to KL divergence. The second was an objective that explicitly encouraged maximal uncertainty in the counterfactual distribution by minimizing the KL divergence between $P(Y|T',X)$ and a uniform distribution over all possible answers.
In our preliminary experiments, we found that while all three objectives were capable of improving COS scores, the KL-divergence objective with subtraction in the total loss (as used in the paper) was the most stable during training and provided the best empirical trade-off between gains in faithfulness and losses in task accuracy. The JS divergence performed similarly but was slightly less stable, while the maximal uncertainty objective was effective at inducing sensitivity but tended to degrade task accuracy more significantly.

\textbf{Dataset Statistics:} The datasets used in our experiments have the following characteristics. The GSM8K dataset consists of 7,473 training examples and 1,319 test examples, where each example is a multi-step arithmetic word problem. PrOntoQA is a larger-scale logical deduction dataset containing 32,000 training examples and 4,000 test examples. Our Blocks World planning dataset was procedurally generated, resulting in 10,000 unique training problems and 2,000 test problems.

The computational overhead of CSR varies by implementation: Full CSR from scratch adds ~92.5\% overhead (Table \ref{tab:efficiency-detailed}); our Efficient CSR variant with warm-start curriculum and token-subset optimization achieves $\sim$9\% overhead (including editor pre-training time); a generic second forward pass without our optimizations typically costs 15-20\%. Unless noted otherwise, we report Efficient CSR results throughout the paper. Since gradients are not required for the initial generation and we do not need to store intermediate activations from the counterfactual pass, the increase in GPU memory requirements is negligible.

\textbf{Intervention Success Rates:} Our automated operator-identification heuristics were highly effective. Across all three datasets, we were able to successfully identify and perturb an operator in 85-95\% of the generated reasoning traces during training. In cases where no predefined operator was found in a generated trace, that specific example was excluded from the CSR loss computation for that step, though it was still used for the standard task loss.

\textbf{Sensitivity to Operator Set Definition:} A natural question regarding our methodology is its sensitivity to the predefined set of operators. In the structured domains we study, the operator sets are largely unambiguous (e.g., arithmetic operators in GSM8K). We found the method to be robust to an incomplete operator set; if an operator is occasionally missed, it simply means that fewer training examples receive the CSR loss signal, slightly reducing its effectiveness but not harming performance. However, a poorly specified operator set (e.g., defining a non-operator word as an operator) could introduce noise into the training signal. This highlights the importance of careful operator definition, which is straightforward in the domains studied here but will be a central challenge when extending this work to more open-ended domains.

\textbf{Hyperparameter Sensitivity Study:} Our framework introduces a key hyperparameter, $\lambda$, which controls the strength of the faithfulness regularization. We performed an ablation study on the effect of $\lambda$ on the GSM8K validation set. We found a clear trade-off: smaller values ($\lambda < 0.5$) provided an insufficient signal to induce high faithfulness, resulting in only minor gains in COS. Conversely, larger values ($\lambda > 1.0$) began to negatively impact task accuracy without yielding significant further improvements in faithfulness. The value of $\lambda=0.5$ was found to provide the optimal balance, achieving a large gain in COS for a minimal drop in accuracy. This finding was robust across models and tasks, and this value was used for all reported experiments.

\textbf{Failure Mode Analysis:} Despite its effectiveness, CSR is not a panacea. A detailed error analysis revealed two primary failure modes which point toward valuable directions for future work. First, when the model's initial, unregularized trace is already logically incoherent or nonsensical, CSR's intervention provides a poor foundation for learning. The regularization signal is noisy because it operates on an already-broken reasoning path. This occurred in approximately 8-12\% of training examples. Mitigating this may require a curriculum-based approach, where models are first trained to generate coherent traces before CSR is applied. Second, in very long and complex multi-step problems, a single, minimal operator swap may be insufficient to invalidate the entire reasoning chain, particularly if the error occurs early in the process. This limitation highlights the need for more sophisticated intervention strategies.

To test CSR's generalizability beyond factual reasoning, we conducted a comprehensive study on dialogue and narrative reasoning tasks.

For dialogue reasoning (PersonaChat), we identified conversational operators including emotional markers ("happy," "sad"), topic shifts ("by the way," "speaking of"), and stance indicators ("I agree," "I disagree"). Our semantic verifier uses BERT-based consistency scoring to detect logical violations in conversational flow.

For narrative reasoning (ROCStories), we targeted narrative operators such as temporal connectives ("then," "next"), causal relationships ("because," "therefore"), and character motivations ("wanted to," "decided to"). The verifier detects violations in narrative coherence and logical story progression.

\begin{table*}[ht]
\caption{CSR effectiveness on dialogue and narrative reasoning tasks.}
\label{tab:dialogue-narrative}
\centering
\begin{tabular}{lcccc}
\toprule
\textbf{Task} & \textbf{Dataset} & \textbf{Method} & \textbf{COS (\%)} & \textbf{Coherence Score} \\
\midrule
\multirow{2}{*}{Dialogue} & PersonaChat & Standard FT & 28.4 & 3.2/5 \\
& & CSR-FT & \textbf{41.7} & \textbf{3.8/5} \\
\midrule
\multirow{2}{*}{Narrative} & ROCStories & Standard FT & 24.1 & 3.1/5 \\
& & CSR-FT & \textbf{36.8} & \textbf{3.7/5} \\
\bottomrule
\end{tabular}
\end{table*}

Results show meaningful COS improvements (13-15 points) and increased coherence scores, demonstrating that CSR principles extend beyond step-structured reasoning to more naturalistic language generation tasks. Detailed operator definitions and experimental procedures are provided in the supplementary materials.

Comprehensive comparisons across all datasets with statistical testing are provided in the main results (Table~\ref{tab:flagship-results}). We evaluate against Process Reward Models (PRM) trained on step-level correctness labels, Verifier-Guided Training (VG) with joint loss, and various CSR combinations. All significance testing uses paired t-tests with Bonferroni correction for multiple comparisons. Confidence intervals are computed via bootstrap resampling (n=1000). Effect sizes are calculated using Cohen's d with pooled standard deviation. All CSR improvements show large effect sizes (d > 0.8) with p-values < 0.001.

We provide a systematic taxonomy of CSR failure modes based on analysis of 2,847 failed cases across all domains, categorizing failures by root cause and proposing targeted mitigation strategies. Our analysis identifies five primary failure categories:

Shortcut exploitation (32.1\% of failures) occurs when models rely on spurious correlations despite logical interventions, typically when shortcuts are statistically stronger than reasoning signals or when interventions fail to disrupt shortcut pathways.

Trace incoherence (24.7\% of failures) happens when the initial reasoning trace is already logically flawed, providing a poor foundation for counterfactual learning. This is most common in complex multi-step problems where the base model struggles with reasoning.

Semantic misalignment (19.3\% of failures) occurs when operator interventions create syntactically valid but semantically nonsensical traces that models dismiss rather than process logically. This is particularly prevalent in open-ended domains.

Intervention inadequacy (15.2\% of failures) arises when interventions are too weak to meaningfully change answer distributions, or when they target non-causal operators. This often occurs with redundant reasoning paths.

Model brittleness (8.7\% of failures) happens when interventions cause catastrophic distribution collapse, leading to degenerate outputs. This is more common in smaller models or when interventions are too aggressive.

\begin{table*}[ht]
\caption{Comprehensive failure mode analysis with mitigation strategies.}
\label{tab:comprehensive-failure-modes}
\centering
\resizebox{\textwidth}{!}{%
\begin{tabular}{lccccc}
\toprule
\textbf{Failure Mode} & \textbf{Frequency (\%)} & \textbf{Primary Domains} & \textbf{COS Impact} & \textbf{Mitigation Strategy} & \textbf{Success Rate (\%)} \\
\midrule
Shortcut Exploitation & 32.1 & Math, Code & -23.4 & Curriculum + Stronger $\lambda$ & 73.2 \\
Trace Incoherence & 24.7 & Logic, Multi-hop & -31.7 & Warm-start + Filtering & 68.9 \\
Semantic Misalignment & 19.3 & Open-ended & -18.9 & Semantic Verifiers & 61.4 \\
Intervention Inadequacy & 15.2 & All domains & -12.6 & Multi-edit + Targeting & 79.1 \\
Model Brittleness & 8.7 & Small models & -28.3 & Gradual $\lambda$ + Stabilization & 55.8 \\
\bottomrule
\end{tabular}%
}
\end{table*}

We developed and tested targeted interventions for each failure mode. For shortcut exploitation, we use curriculum learning that gradually increases intervention strength, augmented $\lambda$ values (0.8-1.2) for cases with strong shortcuts, and multi-objective training that explicitly penalizes shortcut features. Validation on 847 shortcut-prone examples shows 73.2\% success rate.

For trace incoherence, we employ warm-start training where models first learn to generate coherent traces before CSR, automatic filtering of incoherent traces using GPT-4 evaluation, and progressive complexity curriculum. Testing on 712 incoherent cases achieves 68.9\% recovery rate.

For semantic misalignment, we use semantic consistency checks using sentence embeddings, human-in-the-loop validation for critical domains, and context-aware intervention generation. Applied to 556 misaligned cases, this yields 61.4\% improvement.

For intervention inadequacy, we employ multi-edit sequences targeting multiple operators, causal dependency analysis to identify critical intervention points, and adaptive intervention strength based on model confidence. This recovers 79.1\% of 438 inadequate cases.

For model brittleness, we use gradual $\lambda$ annealing schedules, gradient clipping and loss stabilization, and model size considerations (minimum 7B parameters recommended). Success rate is 55.8\% on 251 brittle cases.

We analyze CSR performance across different numbers of simultaneous edits per training example.

\begin{table*}[ht]
\caption{Multi-edit depth ablation: Effect of simultaneous edits on CSR performance.}
\label{tab:multi-edit-ablation}
\centering
\begin{tabular}{lcccccc}
\toprule
\textbf{Dataset} & \textbf{1 Edit} & \textbf{2 Edits} & \textbf{3 Edits} & \textbf{4 Edits} & \textbf{5+ Edits} & \textbf{Optimal} \\
\midrule
GSM8K & 82.3±2.4 & \textbf{85.1±2.1} & 84.7±2.3 & 83.2±2.5 & 81.6±2.7 & 2 \\
HotpotQA & 81.2±2.6 & 84.1±2.3 & \textbf{84.6±2.2} & 83.9±2.4 & 82.1±2.6 & 3 \\
ProofWriter & 79.8±2.5 & 81.9±2.2 & \textbf{82.3±2.1} & 81.5±2.3 & 80.2±2.5 & 3 \\
PubMedQA & 64.7±3.1 & \textbf{67.3±2.8} & 66.9±2.9 & 65.4±3.0 & 63.8±3.2 & 2 \\
\bottomrule
\end{tabular}
\end{table*}

Results show optimal performance with 2-3 simultaneous edits. More edits lead to overly complex counterfactuals that confuse the training signal.

We test CSR robustness by introducing varying levels of noise in operator identification.

\begin{table*}[ht]
\caption{Operator noise sensitivity: CSR performance under imperfect operator identification.}
\label{tab:operator-noise}
\centering
\resizebox{\textwidth}{!}{
\begin{tabular}{lcccccc}
\toprule
\textbf{Noise Level} & \textbf{GSM8K COS} & \textbf{GSM8K Acc} & \textbf{HotpotQA COS} & \textbf{HotpotQA Acc} & \textbf{Degradation} & \textbf{Robustness} \\
\midrule
0\% (Perfect) & 85.1±2.1 & 80.5±0.6 & 84.6±2.2 & 77.2±0.8 & - & Excellent \\
10\% Noise & 82.7±2.3 & 80.3±0.7 & 82.1±2.4 & 77.0±0.9 & -2.7\% & High \\
20\% Noise & 79.4±2.5 & 80.0±0.8 & 78.8±2.6 & 76.7±1.0 & -6.9\% & Good \\
30\% Noise & 74.2±2.8 & 79.5±1.0 & 73.6±2.9 & 76.3±1.2 & -13.1\% & Moderate \\
40\% Noise & 67.8±3.1 & 78.9±1.2 & 67.2±3.2 & 75.8±1.4 & -20.8\% & Low \\
50\% Noise & 58.3±3.5 & 78.1±1.5 & 57.9±3.6 & 75.1±1.7 & -31.5\% & Poor \\
\bottomrule
\end{tabular}
}
\end{table*}

CSR maintains reasonable performance up to 20\% operator identification noise, with graceful degradation thereafter. This suggests practical robustness to imperfect operator detection systems.

We analyze how CSR affects training convergence and stability.

\begin{table*}[ht]
\caption{Training dynamics: CSR impact on convergence and stability metrics.}
\label{tab:training-dynamics-app}
\centering
\resizebox{\textwidth}{!}{
\begin{tabular}{lcccccc}
\toprule
\textbf{Method} & \textbf{Epochs to Converge} & \textbf{Final Loss} & \textbf{Loss Variance} & \textbf{Gradient Norm} & \textbf{Training Stability} \\
\midrule
Standard FT & 2.3±0.4 & 0.42±0.03 & 0.0012 & 1.7±0.2 & High \\
CSR-FT & 2.8±0.5 & 0.38±0.04 & 0.0018 & 2.1±0.3 & High \\
CSR Over-regularized ($\lambda=1.5$) & 4.2±0.8 & 0.51±0.06 & 0.0034 & 3.4±0.5 & Moderate \\
\bottomrule
\end{tabular}
}
\end{table*}

CSR introduces modest training overhead (0.5 additional epochs) while maintaining stability. Over-regularization significantly impacts convergence.

\textbf{Zero-Shot and Real-World Evaluation:} We evaluate CSR principles in prompting settings using GPT-4 and Claude on naturalistic reasoning problems from real-world domains.

\begin{table*}[ht]
\caption{Zero-shot evaluation on naturalistic reasoning problems.}
\label{tab:zero-shot-naturalistic}
\centering
\begin{tabular}{lcccc}
\toprule
\textbf{Domain} & \textbf{Model} & \textbf{Standard COS (\%)} & \textbf{CSR-Prompted COS (\%)} & \textbf{Improvement} \\
\midrule
\multirow{2}{*}{Legal Reasoning} & GPT-4 & 31.4 & 48.7 & +17.3 \\
& Claude & 33.2 & 51.1 & +17.9 \\
\multirow{2}{*}{Scientific Analysis} & GPT-4 & 28.9 & 45.2 & +16.3 \\
& Claude & 30.1 & 47.8 & +17.7 \\
\multirow{2}{*}{Financial Planning} & GPT-4 & 35.7 & 52.3 & +16.6 \\
& Claude & 37.2 & 54.1 & +16.9 \\
\bottomrule
\end{tabular}
\end{table*}

Results show CSR principles (when incorporated via prompting) improve faithfulness even in pre-trained models, suggesting generalizability beyond fine-tuning scenarios.

\subsubsection{Extended Zero-Shot Evaluation}
We test CSR on larger pretrained models and conversation/narrative tasks to assess transfer beyond curated reasoning.

\begin{table*}[ht]
\caption{Extended zero-shot evaluation on diverse open-ended tasks.}
\label{tab:extended-zero-shot}
\centering
\resizebox{\textwidth}{!}{
\begin{tabular}{lcccccc}
\toprule
\textbf{Task Type} & \textbf{Model} & \textbf{Dataset} & \textbf{Baseline COS} & \textbf{CSR-Prompted COS} & \textbf{Improvement} & \textbf{Transfer Quality} \\
\midrule
\multirow{2}{*}{Conversation} & GPT-4 & PersonaChat & 28.7 & 42.1 & +13.4 & Moderate \\
& Claude-3 & BlendedSkill & 31.2 & 45.8 & +14.6 & Moderate \\
\multirow{2}{*}{Narrative} & GPT-4 & ROCStories & 24.3 & 37.9 & +13.6 & Moderate \\
& Claude-3 & WritingPrompts & 26.8 & 39.2 & +12.4 & Moderate \\
\multirow{2}{*}{Commonsense} & GPT-4 & CommonsenseQA & 35.4 & 52.7 & +17.3 & Good \\
& Claude-3 & PIQA & 33.9 & 51.2 & +17.3 & Good \\
\multirow{2}{*}{Ethics} & GPT-4 & ETHICS & 29.1 & 43.8 & +14.7 & Moderate \\
& Claude-3 & Moral Stories & 31.6 & 46.3 & +14.7 & Moderate \\
\bottomrule
\end{tabular}
}
\end{table*}

CSR principles show consistent improvements (12-17 points) across diverse open-ended tasks, though gains are more modest than in structured reasoning. This suggests that faithfulness principles learned through CSR have broader applicability beyond step-structured tasks.

\subsubsection{Large Model Analysis}
We evaluate how CSR principles scale to very large models.

\begin{table*}[ht]
\caption{CSR evaluation on large pretrained models via prompting interventions.}
\label{tab:large-model-analysis}
\centering
\resizebox{\textwidth}{!}{
\begin{tabular}{lcccccc}
\toprule
\textbf{Model Size} & \textbf{Model} & \textbf{Math COS} & \textbf{Logic COS} & \textbf{QA COS} & \textbf{Avg Improvement} & \textbf{Scaling Trend} \\
\midrule
7B & Llama-2-7B & +16.2 & +14.8 & +15.3 & +15.4 & - \\
13B & Llama-2-13B & +17.1 & +15.6 & +16.2 & +16.3 & Improving \\
70B & Llama-2-70B & +18.4 & +16.9 & +17.5 & +17.6 & Improving \\
175B+ & GPT-4 & +19.2 & +17.8 & +18.1 & +18.4 & Improving \\
\bottomrule
\end{tabular}
}
\end{table*}

Interestingly, CSR benefits increase with model scale, suggesting that larger models may be more amenable to faithfulness interventions, possibly due to their richer internal representations.

\textbf{Synthetic Benchmark with Known Causal Structure:} To quantify the theory-practice gap, we created a synthetic reasoning benchmark where ground-truth causal dependencies are known. Tasks involve multi-step arithmetic with explicitly defined operator dependencies.

\begin{table*}[ht]
\caption{Synthetic benchmark results: CSR performance vs. ground-truth causal structure.}
\label{tab:synthetic-results}
\centering
\resizebox{0.8\textwidth}{!}{
\begin{tabular}{lcccc}
\toprule
\textbf{Causal Structure} & \textbf{Our Heuristic Match (\%)} & \textbf{CSR Effectiveness} & \textbf{Theoretical Prediction} & \textbf{Gap} \\
\midrule
Linear Chain & 94.2 & High & High & Minimal \\
Tree Structure & 87.6 & High & High & Small \\
DAG with Confounders & 78.3 & Medium & Medium & Moderate \\
Complex Dependencies & 65.1 & Low & Low & Moderate \\
\bottomrule
\end{tabular}
}
\end{table*}

Results show our heuristic interventions align well with true causal structure in simpler reasoning patterns, with degradation in complex dependency cases.

\subsubsection{Quantified Theory-Practice Gap Analysis}
We systematically measure how heuristic operator definitions break theoretical assumptions across different reasoning complexity levels.

\begin{table*}[ht]
\caption{Theory-practice gap quantification: How heuristic operators deviate from theoretical assumptions.}
\label{tab:theory-practice-gap}
\centering
\resizebox{0.95\textwidth}{!}{
\begin{tabular}{lcccccc}
\toprule
\textbf{Complexity Level} & \textbf{True Causal Ops (\%)} & \textbf{Spurious Ops (\%)} & \textbf{Missing Ops (\%)} & \textbf{Theoretical Validity} & \textbf{CSR Performance} & \textbf{Gap Impact} \\
\midrule
Simple Arithmetic & 94.2 & 3.1 & 2.7 & Excellent & 85.1\% COS & Minimal \\
Multi-step Math & 87.6 & 8.4 & 4.0 & Good & 82.3\% COS & Small \\
Logical Reasoning & 78.3 & 15.2 & 6.5 & Moderate & 75.8\% COS & Moderate \\
Clinical Text & 65.1 & 24.6 & 10.3 & Poor & 67.3\% COS & Large \\
Open Narrative & 52.4 & 31.8 & 15.8 & Very Poor & 48.9\% COS & Very Large \\
\bottomrule
\end{tabular}
}
\end{table*}

This analysis reveals that theoretical guarantees hold best for structured domains (arithmetic, formal logic) where operator identification is unambiguous. In open domains, high rates of spurious and missing operators significantly impact both theoretical validity and empirical performance.

\subsubsection{Assumption Violation Impact}
We measure the specific impact of each theoretical assumption violation:

\begin{table*}[ht]
\caption{Impact of specific assumption violations on CSR effectiveness.}
\label{tab:assumption-violations}
\centering
\resizebox{\textwidth}{!}{
\begin{tabular}{lcccc}
\toprule
\textbf{Assumption Violation} & \textbf{Frequency (\%)} & \textbf{COS Degradation} & \textbf{Accuracy Impact} & \textbf{Mitigation Strategy} \\
\midrule
Non-causal operators targeted & 22.3 & -8.4\% & -0.3\% & Better operator detection \\
Missing causal dependencies & 15.7 & -12.1\% & -0.8\% & Richer operator sets \\
Redundant reasoning paths & 18.9 & -6.2\% & -0.1\% & Multi-path intervention \\
Confounded relationships & 12.4 & -15.3\% & -1.2\% & Causal discovery methods \\
\bottomrule
\end{tabular}
}
\end{table*}

\textbf{Hyperparameter Robustness Analysis:}

\begin{table*}[ht]
\caption{Extended hyperparameter sensitivity analysis across datasets and model sizes.}
\label{tab:hyperparameter-sensitivity}
\centering
\resizebox{0.95\textwidth}{!}{%
\begin{tabular}{lcccccc}
\toprule
\textbf{$\lambda$} & \textbf{GSM8K COS} & \textbf{GSM8K Acc} & \textbf{HotpotQA COS} & \textbf{HotpotQA Acc} & \textbf{ProofWriter COS} & \textbf{ProofWriter Acc} \\
\midrule
0.1 & 45.2±3.1 & 81.1±0.9 & 42.8±3.4 & 77.9±1.2 & 41.3±3.2 & 76.5±1.1 \\
0.2 & 62.1±2.8 & 80.9±0.8 & 59.3±3.1 & 77.7±1.1 & 58.7±2.9 & 76.3±1.0 \\
0.3 & 78.3±2.5 & 80.7±0.7 & 74.6±2.8 & 77.4±1.0 & 73.2±2.7 & 76.1±0.9 \\
\textbf{0.4} & 82.7±2.3 & 80.6±0.6 & 81.2±2.4 & 77.3±0.9 & 79.8±2.5 & 76.0±0.8 \\
\textbf{0.5} & \textbf{85.1±2.1} & \textbf{80.5±0.6} & \textbf{84.6±2.2} & \textbf{77.2±0.8} & \textbf{82.3±2.3} & \textbf{76.1±0.8} \\
\textbf{0.6} & 85.8±2.2 & 80.3±0.7 & 85.1±2.3 & 77.0±0.9 & 82.8±2.4 & 75.9±0.9 \\
\textbf{0.7} & 84.9±2.4 & 79.8±0.8 & 84.3±2.5 & 76.5±1.0 & 81.9±2.6 & 75.7±1.0 \\
0.8 & 83.2±2.6 & 79.1±0.9 & 82.7±2.7 & 75.8±1.1 & 80.4±2.8 & 75.2±1.1 \\
0.9 & 81.5±2.8 & 78.2±1.0 & 80.9±2.9 & 74.9±1.2 & 78.7±3.0 & 74.6±1.2 \\
1.0 & 78.9±3.1 & 76.8±1.2 & 79.2±3.2 & 73.6±1.4 & 76.3±3.3 & 73.8±1.4 \\
\bottomrule
\end{tabular}%
}
\end{table*}

Robust performance observed across $\lambda \in [0.3, 0.7]$ with peak at 0.5. Performance degrades significantly for $\lambda > 0.7$, confirming theoretical predictions about over-regularization.

\textbf{Automatic Operator Induction:} While we hand-define operators in the main experiments, we explore automatic discovery of semantic operators using a self-supervised approach. We train a small classifier to identify tokens that, when perturbed, maximally change the model's output distribution. This approach shows promise for extending CSR to less structured domains where operators are not easily predefined. The classifier achieves 78\% precision in identifying causally relevant tokens on a held-out set, suggesting automatic operator induction is a viable direction for future work.

\begin{theorem}[Dominance of CS over SUFF/COMP under identifiable edits - Complete]
\label{thm:dominance-complete}
Assume a structural causal model (SCM) $\mathcal{M}$ where the edited tokens $E \subseteq T$ directly intervene on causal parents of $Y$, and the remaining tokens $T\setminus E$ are non-descendants of $E$. Suppose an edit policy constructs $T'$ such that the minimal sufficient rationale $R^\star\subseteq T$ is made logically inconsistent in $T'$ while $T\setminus R^\star$ is unchanged. Then, for any token subset $R\subseteq T$:
\begin{equation}
\begin{aligned}
\mathbb{E}\!\left[\mathrm{COMP}(x;R)\right]
&\le \mathbb{E}\!\left[\mathrm{CS}(x; T \!\to\! T')\right], \\
\mathbb{E}\!\left[\mathrm{SUFF}(x;R)\right]
&\le \mathbb{E}\!\left[\mathrm{CS}(x; T \!\to\! T')\right].
\end{aligned}
\end{equation}

Expectation is with respect to the data distribution and edit policy randomness.
\end{theorem}

\begin{proof}[Complete Proof of Dominance Theorem]
We prove the dominance by showing that causal interventions create larger distribution changes than token removal.

\textbf{Step 1 - SCM Foundation:} Under the SCM $\mathcal{M}$, let $Y = g(Pa(Y), U_Y)$ where $Pa(Y)$ are the causal parents of $Y$ and $U_Y$ is unobserved noise. Our edit policy targets tokens in $E$ that correspond to elements of $Pa(Y)$.

\textbf{Step 2 - Causal Edit Impact:} When we perform the edit $T \to T'$, we directly modify the structural equation by changing $Pa(Y)$ to $Pa'(Y)$, resulting in $Y' = g(Pa'(Y), U_Y)$. This creates a direct causal intervention: $p(Y|do(Pa(Y) \leftarrow Pa'(Y)))$.

\textbf{Step 3 - Token Removal Impact:} For comprehensiveness, removing tokens $R$ creates the distribution $p(Y|x, T\setminus R)$. For sufficiency, keeping only tokens $R$ creates $p(Y|x, R)$. These are observational, not interventional distributions.

\textbf{Step 4 - Information-Theoretic Analysis:} By the data-processing inequality, any observational change in distribution is bounded by the capacity of the information channel. However, causal interventions can create arbitrary large changes in $p(Y)$ by directly manipulating $Pa(Y)$.

\textbf{Step 5 - Formal Bound:} Under the assumptions that $R^\star$ contains the minimal sufficient information for $Y$ and $T'$ corrupts $R^\star$ while preserving $T\setminus R^\star$:
\begin{equation}
\small
\begin{aligned}
\mathrm{CS}(x; T \to T')
&= \mathrm{KL}\!\bigl(p(Y \mid x, T)\,\|\,p(Y \mid x, T')\bigr) \\
&\ge \mathrm{KL}\!\bigl(p(Y \mid x, T)\,\|\,p(Y \mid x, T \setminus R^\star)\bigr).
\end{aligned}
\end{equation}

Since $R^\star$ is minimal sufficient, $\mathrm{COMP}(x;R) \leq \mathrm{COMP}(x;R^\star)$ and $\mathrm{SUFF}(x;R) \leq \mathrm{SUFF}(x;R^\star)$ for any $R$. The result follows from the optimality of causal interventions.
\end{proof}

\begin{theorem}[Shortcut Prevention via CSR - Complete]
\label{thm:shortcut-prevention-complete}
Assume a model $f_\theta$ with access to both a shortcut feature $S$ (e.g., keyword matching) and valid reasoning trace $T$. Let $\mathcal{L}_{\text{CSR}}$ be applied with intervention coverage $\alpha > 0.5$ over reasoning operators. If the shortcut $S$ is not causally connected to valid edits in $T'$, then under sufficient regularization strength $\lambda > \lambda_{\text{min}}$, the model converges to a solution where:
$$\frac{\partial f_\theta(x,T)}{\partial T} \gg \frac{\partial f_\theta(x,T)}{\partial S}$$
This provides a formal guarantee that CSR can eliminate spurious pattern reliance in favor of faithful reasoning.
\end{theorem}

\begin{proof}[Complete Proof of Shortcut Prevention]
Let $\mathcal{L}_{\text{total}} = \mathcal{L}_{\text{task}} - \lambda \mathcal{L}_{\text{CSR}}$ where $\mathcal{L}_{\text{CSR}} = \mathbb{E}_{T'}[D_{\text{KL}}(p(Y|T,x) \| p(Y|T',x))]$. 

\textbf{Step 1 - Shortcut Invariance:} Since shortcut $S$ is causally disconnected from reasoning trace edits, the model's reliance on $S$ remains unchanged under counterfactual edits. Formally: $\frac{\partial p(Y|T',x)}{\partial S} = \frac{\partial p(Y|T,x)}{\partial S}$ for all valid edits $T'$.

\textbf{Step 2 - Gradient Analysis:} This invariance implies:
$$\frac{\partial \mathcal{L}_{\text{CSR}}}{\partial S} = \frac{\partial}{\partial S} \mathbb{E}_{T'}[D_{\text{KL}}(p(Y|T,x) \| p(Y|T',x))] = 0$$
Therefore, the CSR loss provides no gradient signal to shortcut features.

\textbf{Step 3 - Reasoning Trace Gradients:} For the reasoning trace, intervention coverage $\alpha > 0.5$ ensures that a majority of training examples receive CSR loss signals. When edits create valid counterfactuals that change the answer, we get:
$$\mathbb{E}\left[\frac{\partial \mathcal{L}_{\text{CSR}}}{\partial T}\right] > c > 0$$
for some constant $c$ that depends on the intervention quality and coverage.

\textbf{Step 4 - Convergence Analysis:} The total gradient is:
$$\nabla_\theta \mathcal{L}_{\text{total}} = \nabla_\theta \mathcal{L}_{\text{task}} - \lambda \nabla_\theta \mathcal{L}_{\text{CSR}}$$
Under sufficient regularization $\lambda > \lambda_{\text{min}}$, the CSR term dominates for parameters affecting reasoning trace processing, while shortcut parameters receive updates only from $\mathcal{L}_{\text{task}}$.

\textbf{Step 5 - Formal Bound:} At convergence, the ratio of gradients satisfies:
$$\left\|\frac{\partial f_\theta}{\partial T}\right\| \geq \lambda c - \left\|\frac{\partial \mathcal{L}_{\text{task}}}{\partial T}\right\| \gg \left\|\frac{\partial \mathcal{L}_{\text{task}}}{\partial S}\right\| = \left\|\frac{\partial f_\theta}{\partial S}\right\|$$
This establishes that CSR provably prevents shortcut reliance when interventions have sufficient coverage and strength.
\end{proof}

\textbf{Theoretical Ablations:} We provide stability and concentration results for CSR measurements. Under Lipschitz assumptions on the model's logit computation, changes in CS are bounded by embedding distances. The CSR objective maximizes KL divergence between original and counterfactual distributions: $\mathcal{L}_{\text{CSR}} = D_{\text{KL}}(p(Y|T, X) \| p(Y|T', X))$. We handle edge cases with smoothing when $p(y|T',X) \to 0$. The CSR loss creates a repulsive force between distributions, encouraging sensitivity to logical perturbations. Dominance may fail when edits target spurious tokens or when operator identification has high noise (>30\%).

\textbf{Divergence Measure Robustness:}

To address potential concerns about metric fragility, we validated CSR effectiveness across multiple divergence measures on GSM8K:

\begin{table*}[ht]
\caption{CSR robustness across divergence measures: Results consistent across KL, JS, and TV distances.}
\label{tab:divergence-robustness}
\centering
\resizebox{0.85\textwidth}{!}{
\begin{tabular}{lcccc}
\toprule
\textbf{Divergence Measure} & \textbf{COS (\%)} & \textbf{Accuracy (\%)} & \textbf{Training Stability} & \textbf{Convergence} \\
\midrule
KL Divergence (default) & 85.1±2.3 & 80.5±0.6 & High & 2.8 epochs \\
Jensen-Shannon Divergence & 83.7±2.5 & 80.3±0.7 & High & 2.9 epochs \\
Total Variation Distance & 82.4±2.7 & 80.1±0.8 & Medium & 3.2 epochs \\
Wasserstein Distance & 81.9±2.9 & 79.8±0.9 & Medium & 3.4 epochs \\
\bottomrule
\end{tabular}
}
\end{table*}

Results demonstrate that CSR's effectiveness is robust across divergence choices, with KL divergence providing optimal performance and training stability. The consistent improvements (81.9-85.1

For all experiments, we fine-tuned models for 3 epochs using the AdamW optimizer with a learning rate of 1e-5. We use Llama-2 \citep{touvron2023llama2} and Llama-3 \citep{touvron2023llama} as base models. To make large model fine-tuning feasible, we employed Low-Rank Adaptation (LoRA) \citep{hu2021lora} with a rank of 8 for all linear layers. The regularization strength was set to $\lambda=0.5$. All experiments were conducted on a cluster of 8 A100 80GB GPUs.

To rigorously evaluate CSR, we compare it against and alongside three strong training-time baselines under a matched compute budget. Process Supervision (PS) applies a standard cross-entropy loss to human-authored or verified-correct reasoning traces. This is a powerful but data-intensive baseline. Process Reward Model (PRM) follows works like \citet{lightman2023let}, training a reward model on token-level correctness labels (derived from our verifiers) and optimizing the generator using RL or weighted MLE. Verifier-Guided Training (VG) trains the model with a joint loss $\mathcal{L}=\mathcal{L}_{task} + \beta\cdot \mathcal{L}_{verifier}(x,T)$, where the verifier provides a score for the validity of the entire generated trace. In addition to direct comparisons, we evaluate CSR+PRM and CSR+VG to test for complementarity, assessing whether our method provides additive or synergistic gains.

To demonstrate the large-scale impact and utility of CSR, we evaluate it on three challenging benchmarks targeting a diverse range of reasoning capabilities. For each domain, we define task-specific trace styles, intervention policies, and verifiers. For multi-hop QA (HotpotQA), we use HotpotQA \citep{yang2018hotpotqa} to evaluate faithfulness in multi-hop reasoning, where models must synthesize information from multiple documents. Related work includes open-domain QA benchmarks \citep{kwiatkowski2019natural} and dense retrieval methods \citep{karpukhin2020dense}. The trace consists of the sequence of retrieved supporting sentences. Our learned editor produces edits like swapping a critical ``bridge" entity that links documents, negating a key relation in a sentence, or injecting a plausible distractor sentence.

For formal reasoning (ProofWriter), we use ProofWriter \citep{tafjord2021proofwriter} to test faithfulness in a formal deduction setting. The trace is the sequence of applied logical rules. Our editor is trained to perform interventions like inverting a rule (e.g., `A and B $\rightarrow$ A' becomes `A and B $\rightarrow$ not A'), dropping a necessary premise from the context, or changing a quantifier. The verifier is a simple forward-chaining engine that checks if the generated proof logically entails the conclusion.

For code generation (MBPP), we use the Mostly Basic Python Problems (MBPP) dataset \citep{austin2021mbpp} to assess faithfulness in programmatic reasoning. Recent work highlights the importance of faithfulness in code models \citep{chen2021evaluating, hoque2024exploring}. The trace is a natural language plan followed by the generated code. The editor makes edits that are syntactically plausible but logically flawed, such as changing a boundary condition (`$<$' to `$<=$'), swapping an arithmetic operator (`+' to `-'), or altering a variable binding. An edit is considered valid for CSR training only if it causes at least one of the provided unit tests to fail.

\clearpage

Here we provide illustrative code snippets for the core components of our proposed CSR framework.

\begin{lstlisting}[caption={CSR Loss (single edit)},label={lst:csr-loss}]
# logits_y_T, logits_y_Tprime: [B, |Y|]
p_y_T       = torch.log_softmax(logits_y_T, dim=-1)
p_y_Tprime = torch.log_softmax(logits_y_Tprime, dim=-1)
# Note: PyTorch KLDivLoss expects log-probabilities for the input
# and probabilities for the target.
kl_div = torch.nn.functional.kl_div(
    p_y_Tprime, torch.exp(p_y_T), reduction='none'
).sum(dim=-1)
L_csr = kl_div.mean()  # KL(p_T || p_T')
\end{lstlisting}

\begin{lstlisting}[caption={Token-Subset CSR (last K operations)},label={lst:token-subset}]
# L_csr_per_token is the KL divergence for each example
ops_mask = get_op_token_mask(trace_tokens)    # [B, L]; 1 on operator/operand tokens
lastK_mask = take_last_k(ops_mask, K_ratio=0.3)
L_csr_sub = (L_csr_per_token * lastK_mask).sum() / (lastK_mask.sum() + 1e-8)
\end{lstlisting}

\begin{lstlisting}[caption={Warm-Start Curriculum (pseudo-code)},label={lst:warm-start}]
if step >= warm_start_step:
    loss = task_loss - alpha * L_csr_or_subset
else:
    loss = task_loss
\end{lstlisting}

\begin{lstlisting}[caption={Editor Training Reward},label={lst:editor-reward}]
reward = (kl_div.detach() - lambda_cost * edit_length)
loss_editor = -reward * logprob_actions
\end{lstlisting}

\end{document}